\documentclass{article}

\usepackage[preprint]{neurips_2026}

\usepackage[utf8]{inputenc} 
\usepackage[T1]{fontenc}    
\usepackage{hyperref}       
\usepackage{url}            
\usepackage{booktabs}       
\usepackage{amsfonts}       
\usepackage{nicefrac}       
\usepackage{microtype}      
\usepackage{xcolor}         
\usepackage{amsmath} 
\usepackage{graphicx}
 \usepackage{amssymb}
 \usepackage{mathtools}
 \usepackage{amsthm}
 \usepackage{mathrsfs} 
 \usepackage{subcaption}
\usepackage[capitalize,noabbrev]{cleveref}  
\usepackage{bigints}
\usepackage{enumitem}
\usepackage{booktabs}
\usepackage{multirow}
\usepackage{todonotes}
\usepackage{xspace}

\usepackage{algorithm} 
\usepackage{algorithmic} 
\usepackage{algorithmic} 
\allowdisplaybreaks

\newcommand{\R}{\mathrm{R}}
\newcommand{\Leb}[1]{\mathit{L}_{#1}}
\newcommand{\Rel}{\mathbb{R}}

\newcommand{\Exp}{\mathbb{E}}

\newcommand{\z}{\mathbf{z}}

\def\argmin{\mathop{\rm arg\,min}}

\renewcommand{\H}{\mathcal{H}}
\renewcommand{\P}{\mathbb{P}}
\newcommand{\A}{\mathscr{A}}
\renewcommand{\S}{\mathcal{S}}

\usepackage{thm-restate}
\newtheorem{assumptionalt}{Assumption}[section]

\newenvironment{assumptionp}[1]{
  \renewcommand\theassumptionalt{#1{$^\prime$}}
  \assumptionalt
}{\endassumptionalt}

\crefname{assumptionalt}{assumption}{assumptions}
\crefname{assumptionp}{assumption}{assumptions}
\crefname{assumption}{assumption}{assumptions}
\crefname{lemmma}{lemma}{lemmas}
\crefname{proposition}{proposition}{propositions}
\crefname{corollary}{corollary}{corollaries}
\crefname{lemma}{lemma}{lemmas}
\crefname{theorem}{theorem}{theorems}

\newcommand{\indic}[1]{\mathbb{I}\left\{#1\right\}}
 
\theoremstyle{plain}
\newtheorem{theorem}{Theorem}
\newtheorem{example}{Example}

\newtheorem{lemma}{Lemma}

\theoremstyle{definition}
\newtheorem{definition}{Definition}
\newtheorem{assumption}{Assumption}
\theoremstyle{remark} 
\newtheorem{remark}[theorem]{Remark} 

\title{Reinforcement learning with an expectile-based objective}
\author{%
  \textbf{Shrey Rakeshkumar Patel}$^1$ \quad
  \textbf{Sumedh Gupte}$^2$ \quad
  \textbf{Soumen Pachal}$^2$ \\
  \textbf{Prashanth L.~A.}$^1$ \quad
  \textbf{Sanjay P. Bhat}$^2$ \\[0.5ex]
  $^1$Indian Institute of Technology Madras \quad
  $^2$TCS Research, India 
}

\begin{document}

\maketitle

\begin{abstract}  

We consider the policy evaluation and control in a finite horizon reinforcement learning (RL) setting under an expectile-based objective. First, we derive the mean-squared error (MSE) and concentration bounds for the classic estimator of expectiles based on independent and identically distributed (i.i.d.) samples. To the best of our knowledge, expectiles have not been analyzed in the non-asymptotic regime and the bounds we derive may be of independent interest.  Next, we analyze a Monte Carlo type estimator of expectile of the Markov chain underlying a given policy. We derive upper bounds that hold in expectation as well as with high probability for this estimator. Further, we show the order-optimality of our estimator by deriving a lower bound on expectile-based policy evaluation.  For the problem of control, we adopt a policy gradient approach and derive a policy gradient theorem for expectiles. Using this result, we propose a gradient estimator with a $O\left(1/m\right)$ mean-squared error bounds, where $m$ is the number of trajectories. Further, under standard assumptions for policy gradient-type algorithms, we establish smoothness of the expectile-sensitive objective, in turn leading to stationary convergence rate bounds for the overall risk-sensitive policy gradient algorithm that we propose. Finally, we conduct numerical experiments to show the utility of expectiles on popular RL benchmarks.

\end{abstract}


\section{Introduction}
\label{sec:intro}
In recent years, risk-sensitive reinforcement learning (RL) has received increasing research attention, cf. \cite{borkar2001sensitivity,kose2021risk,prashanth2015cumulative,markowitz2023risk,borkar2010learning,prashanth2016variance,Mihatsch02RS,tamar2012policy,mihatsch_risk-senitive_2002,bauerle_markov_2024} for some representative works. The shortcomings of expected value as an objective for decision-making and the need to incorporate risk are well known \cite{garcia_comprehensive_2015,biswas_ergodic_2023}, especially in domains such as finance, robotics, and transportation. 
From a practical standpoint in the context of RL, risk-sensitive methods offer the decision-maker the option to obtain policies that look beyond a scalar mean reward and optimize richer metrics. For instance, in safety applications, the tail of the loss distribution carries more importance than the average loss. 
Several risk measures have been studied in an RL framework, e.g., variance, quantiles, Conditional Value-at-Risk (CVaR), and cumulative prospect theory.  
Coherent risk measures \cite{tamar_policy_2015,bellini2017risk} are a popular class of risk measures that possess desirable properties such as monotonicity, sub-additivity, translation invariance, and positive homogeneity. In a financial application, coherence ensures that diversification does not increase risk.
 Expectiles, introduced by \citep{newey1987asymmetric},  fall under the realm of coherent risk measures, unlike quantiles. In fact, expectiles are the only coherent risk measures that are elicitable \citep{Bellini_2015}. 
 
 While there are sufficient arguments in the literature for using coherent risk measures \citep{follmer2002convex}, a thorough technical analysis specialized to expectiles remains underexplored in the RL literature, and our work bridges this gap.
More specifically, we consider the problems of policy evaluation and control in a typical RL setting, with an objective that is an expectile.
To motivate this risk-sensitive objective, we present a simple four-state chain with Gaussian single-stage costs and show that adopting an expectile objective yields a policy that achieves a trade-off between mean and standard deviation. 
 

We now summarize our contributions below. 

\textbf{Expectile estimation:} We first analyze the classical estimator of expectiles based on independent and identically distributed (i.i.d.) samples and derive MSE and concentration bounds for this estimator. These bounds are of independent interest.

\textbf{Policy Evaluation:}  We employ a Monte Carlo-type estimator of the expectile of the Markov chain underlying a given policy. We derive an upper bound for this estimator in expectation and with high probability. We also provide a lower bound that shows the order-optimality of this estimator for a fixed finite horizon. 

\textbf{Control:} 
A prominent solution approach for optimizing risk measures in an RL context involves policy gradient algorithms; see \citet{prashanth2022risk}. We adopt this approach and derive a policy-gradient theorem, albeit with an expectile objective in a finite-horizon MDP. 
Next, we form an estimator of the expectile gradient using sample trajectories of the underlying MDP and establish mean-squared error bounds for this estimator.   
We incorporate this estimator in a stochastic gradient algorithm and establish convergence to an approximate first-order stationary point in the non-asymptotic regime. 

\textbf{Experiments:} We implement our proposed algorithm in several MuJoCo environments and compare it with  REINFORCE \cite{Williams1992}, CVaR, and mean-variance risk as baselines. From our experiments, we observe that the expectile outperforms all the baselines by exhibiting better average rewards and lower variance.  




\textbf{Related work.} 
The risk-neutral RL literature has seen a steady development, driven by foundational works that gave the REINFORCE algorithm, the policy gradient theorems, asymptotic and non-asymptotic first-order stationary convergence (FOSC) rates, upper and lower bounds on the unbiased estimator of the gradient, variance-reduction techniques, and so on. Our work is foundational and theoretical in nature and spans all the aforementioned areas. Expectiles have been studied for several decades, see \citet{taylor2008estimating,newey1987asymmetric,bellini2017risk,delbaen2013remark,Bellini_2015,ziegel_coherence_2016,philipps_interpreting_2021,daouia_expectile_2024,adam2025nonparametric}.
Within RL, expectiles have received comparatively limited attention; see \citet{clavier2024bootstrapping,rowland2019statistics,jullien2023distributional}. In \cite{clavier2024bootstrapping}, the authors formulate a expectile variant of the Bellman equation and derive algorithms for finding a policy that solves this equation.  In \cite{rowland2019statistics,jullien2023distributional}, distributional RL is studied with expectiles. Expectiles also appear as a computational tool in offline RL. The authors of IQL \citep{kostrikov2021offline} consider the
offline RL setting where one performs policy improvement using
the target $r + \gamma \max_{a'} Q(s',a')$. This is challenging
as not all pairs $(s',a')$ are available in the offline dataset,
and they propose to approximate the maximization using an upper
expectile of the action values at each state, with $\nu$ close
to $1$, i.e., by solving an \emph{expectile estimation} problem.
They consider a risk-neutral objective in an offline RL setting,
whereas our work considers a risk-sensitive objective in the
online RL setting, i.e., we consider \emph{expectile optimization}
over the policy parameters.

Related approaches consider dynamic convex risk measures, expected conditional risk measures, and distortion risk measures. Coherent risk measures and their policy gradient estimator in \cite{tamar_policy_2015} is obtained by solving $m$-dimensional convex optimization problem where $m$ is the sample size. Our gradient estimator only requires risk estimation, which is a scalar optimization problem. The expectile estimation problem can be formulated as a root-finding problem and solved using a bisection-style algorithm. A $\Phi$-Divergence-Risk (PDR) is studied in \cite{ni2025risk}, but a direct gradient expression is not provided. Instead, they use the dual representation of coherent risk measures to formulate a Lagrangian and derive policy-gradient expressions for the partial derivatives of this Lagrangian. A solution to this problem is computationally more intensive than our direct approach to expectile gradient estimation. Both \cite{tamar_policy_2015} and \cite{ni2025risk} provide asymptotic guarantees on their gradient estimators, while we provide non-asymptotic guarantees. 
The aforementioned references do not consider a policy gradient approach for optimizing expectiles, and the policy obtained by their algorithm does not necessarily optimize the expectile of the cumulative cost random variable. To the best of our knowledge, bounds for expectile estimation in the i.i.d. and Markovian cases have not been derived in the literature; our work fills this gap.

Beyond these uses of expectiles, policy gradient methods have also been studied for other risk-sensitive objectives. In \cite{coache2024reinforcement}, a policy gradient algorithm for dynamic convex risk measures is presented, which requires solving a Lagrangian at each step. They provide neither asymptotic guarantees on this gradient estimator nor convergence guarantees for their algorithm. In contrast, we provide an $\epsilon$-FOSP convergence of our proposed algorithm. Expected Conditional Risk Measures (ECRM) \cite{yu2023global}, a class of coherent risk measures with a linear utility function, such as CVaR, have been studied using the gradient domination property. This leads to a global asymptotic convergence of their policy gradient update step. We note that the convergence is established for the true gradient expression, not for the gradient replaced by a sample estimate, and their work does not provide an analysis of such a sample-based gradient estimate. Distortion risk measures have been studied in \cite{vijayan2021policy} and provide a non-asymptotic analysis without accompanying lower bounds. These risk measures do not overlap with expectiles. 

The rest of the section is organized as follows. In \Cref{sec:prob-formulation}, we present the notation and preliminaries for Markov decision processes and the expectile risk measure. In \Cref{sec:expectiles}, we provide an analysis of policy evaluation methods for expectiles. In \Cref{sec:policy_optimization}, we present the policy gradient theorem for expectiles. Then, we propose and analyze an expectile gradient estimator. Subsequently, we present our expectile-sensitive policy gradient theorem. In \Cref{sec:experiments}, we demonstrate the performance of our algorithm on popular RL benchmarks, including some well-known MUJOCO tasks. Appendix \ref{properties: expectile_properties} contains the detailed proofs of all the results presented in this work.

\section{Problem formulation}\label{sec:prob-formulation}
We define a finite-horizon Markov decision process (MDP) as the tuple $\left<S,A, P, C, \gamma\right>$, where $S$ and $A$ denote the state space and the action space, respectively. We use the symbol $\Delta$ to denote a probability distribution. For instance, given a set $A$, $\Delta(A)$ denotes a probability distribution over $A$. The sets $S$ and $A$ are compact, and the horizon length is given by $T$. The MDP transition kernel is defined as $P: S \times A \to \Delta(S)$, where $P(s'| s, a)$ denotes the probability of going to state $s'$ after taking action $a$ in state $s$ and $\Delta(S)$ denotes a probability distribution over $S$.

 The MDP cost function is defined as $C: S \times A \times S \to \Delta(\R)$ respectively, and $\gamma \in [0,1]$ denotes the discount factor. The initial state distribution is denoted by $P_0$. 

The trajectory space is denoted by $\mathcal{T}$, where every $\tau \in \mathcal{T}$ is a set of states, actions and rewards, denoted as $\tau = \left<s_0,a_0,r_1,s_1,a_1,r_2,s_2, \ldots, s_{T-1}, a_{T-1}, r_T, s_T \right>$. The cumulative discounted cost $c(\tau)$ corresponding to trajectory $\tau$ is given by $c(\tau) = \sum_{t=0}^{T-1}\gamma^{t} C(s_t,a_t,s_{t+1})$. 

 We consider the set of stochastic policies denoted as $\Pi$, where each $\pi: S \to \Delta(A)$ in $\Pi$ is Markovian and stationary. We assume that the policy set $\Pi$ is parameterized by $\Theta \subseteq \mathbb{R}^d$, a convex and compact subset in $\mathbb{R}^n$. Let $\theta \in \Theta$. Then $\pi_\theta:S \to \Delta(A)$ denotes a policy parameterized by $\theta$, and $p_\theta(\tau)$ denotes the probability of observing the trajectory $\tau$ under policy $\pi_\theta$. We define the score function $g:\Theta \times \mathcal{T} \to \Rel^d$ as
 \begin{equation}\label{eq:grad-log-definition}
    g(\theta,\tau) = \sum_{t=0}^{T-1}\nabla_\theta \log \left(\pi_\theta\left(a_t|s_{t}\right)\right).
\end{equation}

Let $(\mathcal{T}, \Sigma)$ be a measurable space where $\Sigma$ is a $\sigma$-algebra on $\mathcal{T}$. The cumulative discounted cost of the MDP is given by the random variable $F$, defined on $(\mathcal{T}, \Sigma)$. Each policy $\pi_\theta$ induces the measure $p_\theta$ on $F$, and $F$ takes the value $c(\tau)$ with probability $p_\theta(\tau)$. Throughout this paper, $G(\theta)$ denotes a random vector on $(\mathcal{T}, \Sigma)$ that takes the value $g(\theta,\tau)$ with probability $p_\theta(\tau)$. Observe that $\Exp_\theta\left[G(\theta)\right] = 0$. 

 \textbf{Risk-neutral objective.} The expected cumulative discounted-cost associated with a policy $\pi_\theta$ is denoted by $J(\theta)$, where $J(\theta) = \Exp_\theta\left[F\right] =  \int_\mathcal{T} c(\tau) p_\theta(d \tau)$. Since the parameter $\theta$ affects the probability measure, we denote the expectation as $\Exp_\theta$. The risk-neutral RL objective is given by
\begin{align*}
    &\text{find } \;\argmin_{\theta \in \Theta} J(\theta).
\end{align*}

\textbf{Risk-sensitive objective.} We aim to find policies that are risk-optimal, with expectile as the underlying risk measure.
We first define expectile for a general random variable $X$ below\footnote{Throughout this paper, $X$ is a random variable that denotes costs, i.e., lower values are preferable. For instance, in a financial application, $X$ could be the loss associated with a financial position.}.
\begin{definition}    
    For a given $\nu \in (0,1)$, the expectile $\xi_\nu$ of $X$ is defined as follows:
    \begin{equation} \label{eq:expectile}
        \xi_\nu = \argmin_{k} \mathcal{L}_\nu(k), \textrm{  where } \mathcal{L}_\nu(k) \!=\!  \Exp\left [e_{\nu}(X - k)\right ],
    \end{equation}
    with $e_{\nu}(x) = x^2 \left|\nu - \indic{x \le 0}\right|$. Here $\indic{\cdot}$ denotes the indicator function.
\end{definition}
For $\nu = 0.5$, expectile coincides with the mean. In risk-sensitive optimization, $\nu > 0.5$ corresponds to  a risk-averse objective, while $\nu < 0.5$ implies risk-seeking. 
Further, the loss function $\mathcal{L}_\nu(k)$ used to define the expectile is continuous and differentiable on $\Rel$.

Let $\Leb{0}$ denote the class of bounded and real-valued random variables. For every $\theta \in \Theta$, let $\rho_\theta: \Leb{0} \to \R$ denote a risk measure under probability measure $p_\theta$, that assigns a real value to a random variable in $\Leb{0}$. Then, the associated risk-sensitive objective is to find a $\theta_{\exp}^*$ such that  
\begin{align} 
\label{eq:objective}
    &  \theta_{\exp}^* \in \;\argmin_{\theta \in \Theta} \xi_{\nu}(\theta), \;\text{where} \;\xi_{\nu}(\theta) = \rho_\theta\left[F\right].
\end{align}
\begin{figure}[htbp]
    \centering
    \begin{subfigure}[b]{0.32\textwidth}
        \centering
        \includegraphics[width=\textwidth]{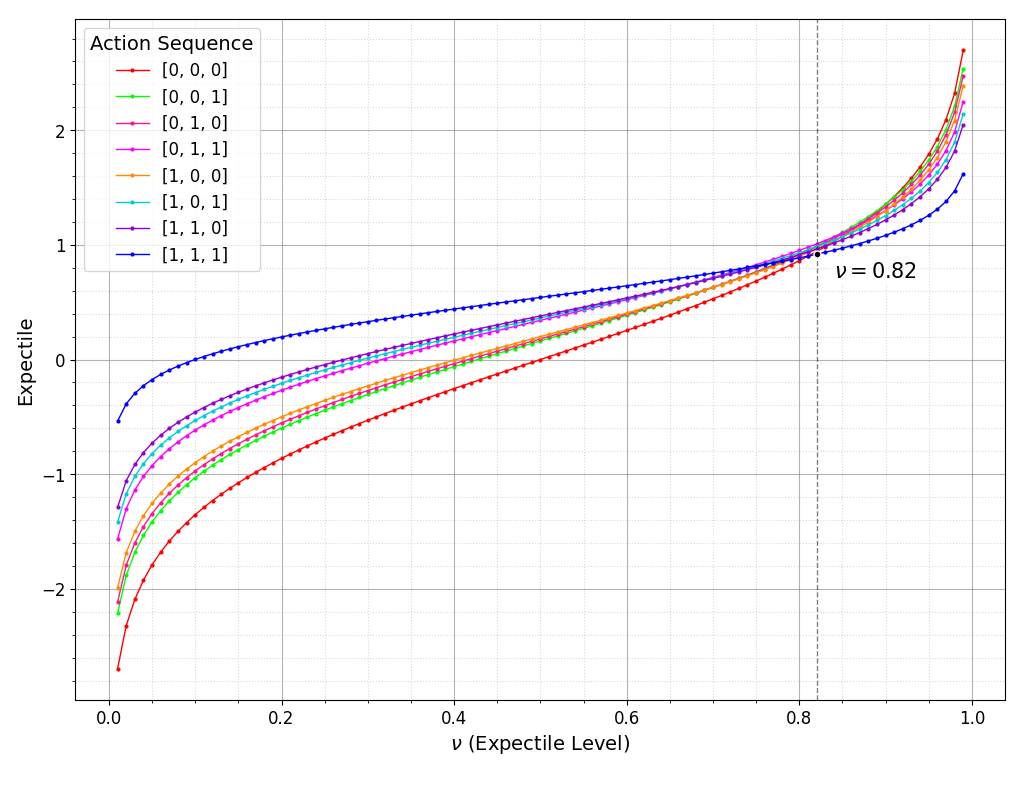}
         \caption{Expectile cost at state $x_0$}
        \label{fig:img1}
    \end{subfigure}
    \hfill
    \begin{subfigure}[b]{0.33\textwidth}
        \centering
        \includegraphics[width=\textwidth]{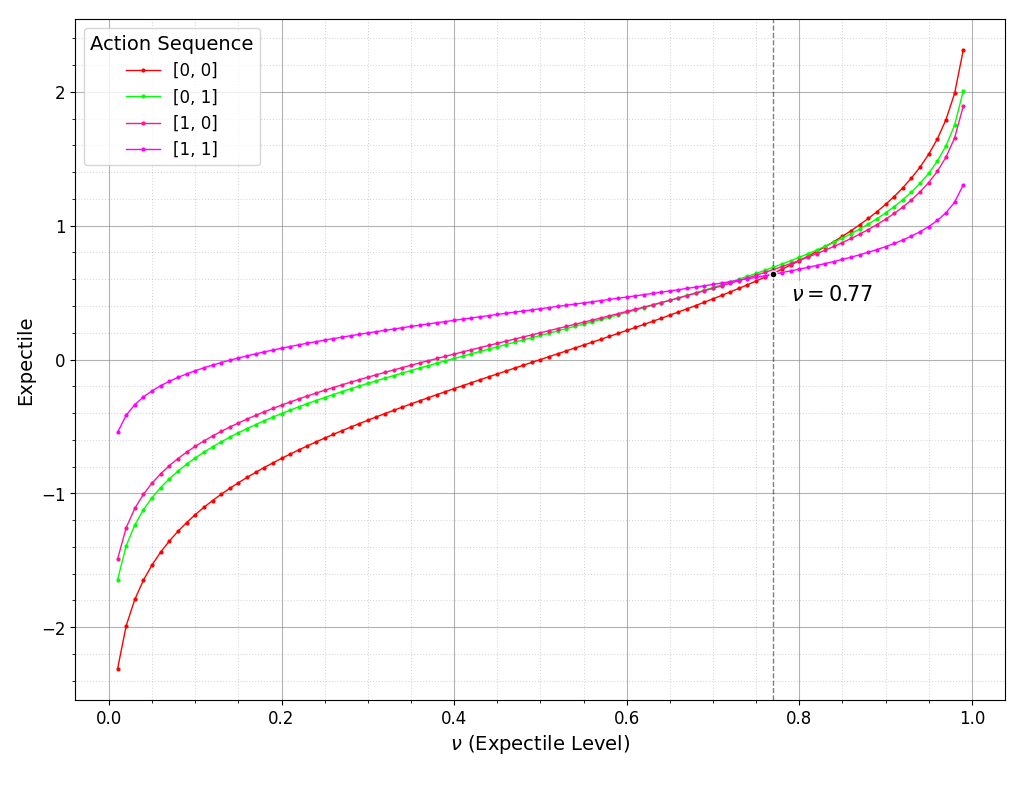}
         \caption{Expectile cost at state $x_1$}
        \label{fig:img2}
    \end{subfigure}
    \hfill
    \begin{subfigure}[b]{0.33\textwidth}
        \centering
        \includegraphics[width=\textwidth]{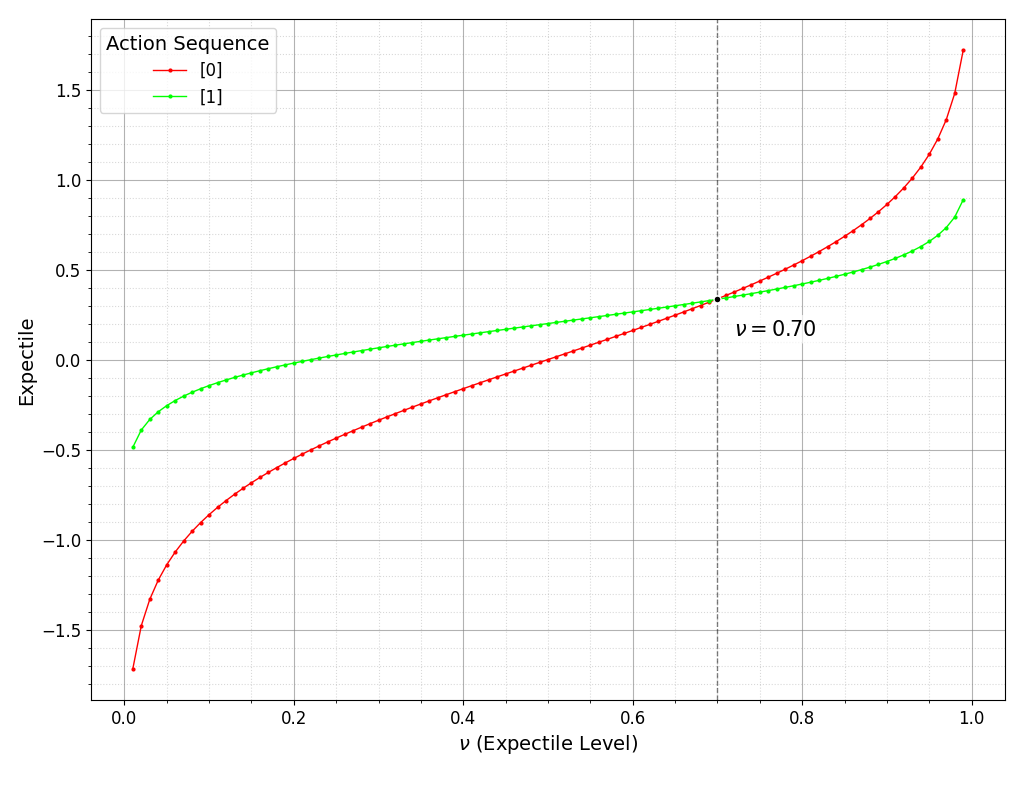}
         \caption{Expectile cost at state $x_2$}
        \label{fig:img3}
    \end{subfigure}
    \caption{Expectile value of the cumulative cost for stationary deterministic policies as a function of the expectile level $\nu\in (0,1)$. In states $x_0, x_1$ and $x_2$, there are eight, four and two action sequences, respectively.}
    \label{fig:4-states-mdp-deter}
\end{figure}

\begin{example}  
\label{exm:four_state} 
    Adapted from \cite{NEURIPS2022_c88a2bd0}, we consider a simple four-state MDP designed to have a known optimal policy for any expectile levels. Unlike their original formulation which relies on rewards, we adapt the environment to use costs. The state sequence begins at $x_0$ and concludes at a terminal state $x_3$. In each transient state $x_i \in \{x_0, x_1, x_2\}$, the agent chooses between two actions, $a_0$ and $a_1$. Both actions cause a transition of the state from $x_i$ to $x_{i+1}$. The stochastic immediate cost depends exclusively on the chosen action: selecting $a_0$ incurs a cost drawn from a standard normal distribution, $\mathcal{N}(0, 1)$, whereas selecting $a_1$ incurs a cost drawn from $\mathcal{N}(0.2, 0.4^2)$. For this task, we use $\gamma = 0.9$. $a_0$ gives a lower cost but with higher variance. 

    We study this problem with an expectile-based objective. As the immediate costs for both $a_0$ and $a_1$ are drawn from independent Gaussian distributions, the discounted cumulative cost for any stationary deterministic policy from transient states is a linear combination of independent Gaussians, which in turn guarantees that the discounted cumulative cost follows a Gaussian. As established in \Cref{lemma:expectile-transition-property-normal}, the $\tau$-expectile of a Gaussian random variable $Z \sim \mathcal{N}(\mu, \sigma^2)$ can be analytically expressed as $\mu + k_\nu\sigma$, where $k_\nu$ is a scalar multiplier determined solely by the expectile level $\nu$. This closed-form relationship shows that evaluating the expectile in this setting inherently translates to optimizing for the mean $\mu$ and the standard deviation $\sigma$. To provide intuition for how the optimal policy changes with $\nu$, we can examine the expectile of the cumulative cost for every available trajectory. For example, in state $x_1$, if $\nu \ge 0.77$, the action sequence $[1, 1]$ is more desirable, whereas for $\nu \le 0.76$, the sequence $[0, 0]$ is more desirable. A similar analysis can be applied to the other states. For a general intuition on these dynamics across states $x_0$, $x_1$, and $x_2$, please refer to \Cref{fig:4-states-mdp-deter}.

    Moreover, it can be shown that for every strict stationary randomized policy, i.e., one which picks action $a_0$ with some probability $p$ and $a_1$ with $(1-p)$, for $p\in (0,1)$, there exists a stationary deterministic policy whose expectile does not exceed that of the randomized policy. This shows that a stationary deterministic optimal policy exists. Please refer to \Cref{lemma:expectile-relation-gmm} in the Appendix \ref{properties: expectile_properties} for the formal proof.
    
\end{example} 

\section{Policy evaluation for expectiles} \label{sec:expectiles}

Recall that an expectile at level $\nu$ is the minimizer of $\mathcal{L}_\nu(k) = \mathbb{E}[e_\nu(X-k)]$. First, we claim that $\mathcal{L}_\nu(k)$ is $\mu_\nu$ strongly-convex and $ L_\nu$-smooth function of $k$, where
\begin{equation}
    \mu_\nu = 2\min(\nu, 1-\nu), \textrm{ and } L_\nu = 2\max(\nu,1-\nu).\label{eq:expectile-const}
\end{equation}
We refer to the \Cref{properties: expectile_properties} for the proof of this claim.
Consequently, the expectile $\xi_\nu$ exists and is unique.

The expectile $\xi_{\nu}$ can alternatively be characterized as the unique solution to the following (see Theorem 1 of \cite{newey1987asymmetric}): 
\begin{equation} \label{eq:Expectile_identification_equation}
\Exp[l_{\nu}(X-\xi_\nu)] = 0, \text{where }l_{\nu}(x)= \nu x\indic{x > 0} + (1 - \nu) x\indic{x \le 0 }.
\end{equation}

\subsection{Expectile estimation from i.i.d. samples}
Let $X_1, X_2, \dots, X_m$ be i.i.d. copies of a random variable $X$ with distribution function $H$. The empirical $\nu$-expectile $\hat{\xi}_{\nu}^m$ is obtained as the solution to 
\begin{equation}
    \label{eq:empirical_expectile_identification_equation}
    \frac{1}{m}\sum_{i=1}^{m} l_\nu(X_i - \hat{\xi}_{\nu}^m) = 0.
\end{equation}

We next present non-asymptotic bounds on the expectile estimate defined above.
\begin{theorem} \label{thm: mse_bound_expectile}
    Suppose $X$ has a finite second moment. 
    Then, we have the following bound for $\hat{\xi}_{\nu}^m$ formed using \eqref{eq:empirical_expectile_identification_equation}:
\begin{equation}
    \Exp[(\xi_\nu -\hat{\xi}_{\nu}^m)^2] \leq \frac{L_\nu^2}{m \mu_\nu^2} \Exp[(X-\xi_\nu)^2],
\end{equation}
    In addition, if $X$ is sub-Gaussian\footnote{A random variable $X$ is $\sigma$-sub-Gaussian if $\Exp[\exp(\lambda(X - \Exp[X]))] \leq \exp(\frac{\lambda^2\sigma^2}{2})$, for all $\lambda \in \mathbb{R}$.} with parameter $\sigma$, then
\begin{equation}
    \P\left[ |\xi_\nu -\hat{\xi}_{\nu}^m| \ge \epsilon \right] \leq 2\exp{\left(- \frac{m\epsilon^2\mu_\nu^2}{8 L_\nu^2 \sigma^2}\right)},
\end{equation}
where $\mu_\nu, L_\nu$ are given in \eqref{eq:expectile-const}.
\end{theorem} 
\begin{proof}
See Appendix~\ref{proofs: mse_bound_expectile}.
\end{proof}

\subsection{Policy evaluation for an expectile}
The need for policy evaluation naturally leads to the study of Markov Cost Processes (MCPs). Once a specific policy is fixed, the system's dynamics simplify into a standard Markov chain paired with a cost structure. Following the formulation in \cite{thoppe2023risk}, we consider a Markov cost process (MCP), which is a tuple $(M, C)$. Here $M \equiv (\S, P, \upsilon)$ is a Markov chain and $C$ is a cost function. The cumulative discounted cost over a $T$-length horizon is the random variable:
\begin{equation}
    Z = \sum_{t=0}^{T-1} \gamma^t C_t, \quad \text{where } C_t \sim C(s_t) \text{ and } \gamma \in [0, 1).
\end{equation}

\paragraph{Lower bound.}

Let $\xi_\nu(M, C)$ denote the expectile of the cumulative discounted cost $Z$ at level $\nu \in (0, 1)$.
Let $\kappa_\nu$ denote the $\nu$-expectile of a standard
Gaussian. For a fixed finite constant
$K\ge 1+\kappa_\nu^2$, let $\H_k$ denote the class of MCPs with horizon $T$ satisfying a certain moment bound given by 
\[
\mathcal H_K := \left\{ (M,C) \in \mathcal H : \mathbb{E}\left[(Z-\xi_\nu(Z))^2\right] \le K \right\}.
\]

 We are interested in estimating this quantity given a sample path of the underlying Markov chain.
An estimate of the expectile $\hat \xi_\nu$ is formed from a sample path of length $n$ of the underlying Markov chain, which is periodically restarted from the initial distribution at intervals of $T$ steps to guarantee mutually independent sequences. More precisely, let $H_n = (s_0, C_0, \dots, s_{n-1}, C_{n-1})$, where $s_i\in \mathcal{S}$ are the states visited, and $C_i$ the corresponding cost sample. Because of these forced restarts, the concatenated sequence $H_n$ provides strictly independent trajectories of length $T$. Thus, an expectile estimator is a map $\hat{\xi}_{\nu, n}: H_n \to \mathbb{R}$.


We first provide a lower bound that characterizes the fundamental limit on the expectile estimation problem in a Markovian framework. We build upon the foundational singleton construction of \cite{thoppe2023risk}, but our proof deviates in a non-trivial way. Unlike \cite{thoppe2023risk}, who rely on a complex $2n$-sample hypothetical mapping to handle infinite-horizon limits, in our proof, we directly compute the KL divergence for a finite horizon $T$, thereby bounding it precisely to $2\epsilon^2$. Furthermore, we employs a custom, optimized Markov inequality substitution to establish its expectation bound.

\begin{theorem}
    \label{thm:lower_bound_mcp_expectile}
    Fix $\nu\in(0,1)$ and a finite constant
$K\ge 1+\kappa_\nu^2$.
    For an MCP $(M,C) \in \H_K$, Let $\xi_\nu(M,C)$ be the expectile of an MCP $(M,C) \in \H_K$. 
    Then, for every sample path length $n \in \mathbb{N}$ and  for any algorithm $\A$ that outputs an expectile estimate $\hat{\xi}_{\nu, n}$ using a $n$-length sample path, we have
    \begin{align} \label{eq:probability_minimax_bound_mcp_expectile}
        \inf_{\A} \sup_{\substack{(M,C)\\\in \H_K}} \mathbb{P}\left( \left|\hat{\xi}_{\nu, n} - \xi_\nu(M,C)\right| \ge \frac{\ln{\left(\frac{1}{2\delta}\right)}}{2\sqrt{n}} \right) \ge \delta,
    \end{align}
    for any $\delta \in \left[\frac{1}{2e}, \frac{1}{2}\right )$.
    In addition,
    \begin{equation} \label{eq:MAE_minimax_bound_mcp_expectile}
        \inf_{\A} \sup_{(M,C) \in \H_K} \mathbb{E}\left[ \left|\hat{\xi}_{\nu,n} - \xi_\nu(M,C)\right| \right] \ge \frac{3}{32\sqrt{n}}.
    \end{equation}
\end{theorem} 
\begin{proof}
See Appendix~\ref{proofs: lower_bound_mcp_expectile}.
\end{proof}




\paragraph{Upper bound.}
We estimate the expectile of the cumulative discounted cost $Z$ for a discounted Markov chain with a finite horizon $T$. More precisely, given a total budget of $n > T$ transitions under an episodic sampling model with periodic resets after every $T$ steps, we obtain $m = \left\lfloor \frac{n}{T} \right\rfloor$ independent trajectories. Any remaining $n-mT$ transitions are discarded. Let $Z_1,\dots, Z_m$ denote the cumulative discounted cost samples obtained from these $m$ trajectories over the finite horizon $T$. With these samples, we form an estimate $\hat{\xi}_\nu^m$ of the expectile of the cumulative discounted cost random variable $Z$ with \eqref{eq:empirical_expectile_identification_equation}.




\begin{theorem}
\label{theorem: Upper_bound_for_MCP_expectiles}
Fix $\nu$.
    Assume the cumulative discounted cost $Z$ has a finite second moment. Let $\hat{\xi}_\nu^m$ denote expectile estimator defined by (\ref{eq:empirical_expectile_identification_equation}) using sample path of $n$ transitions and horizon length $T$. Let $m = \left\lfloor \frac{n}{T} \right\rfloor$. Then, we have
    \begin{equation} \label{eq:Upper_bound_for_MCP_expectiles}
    \Exp |\hat{\xi}_\nu^m - \xi_\nu(Z)| \le \frac{L_\nu}{\sqrt{m} \mu_\nu} \sqrt{\Exp[(Z-\xi_\nu)^2]}.
    \end{equation}

    In addition, if $Z$ is $\sigma$-sub Gaussian, then
    \begin{equation} \label{eq: Concentration_upper_bound_for_MCP_expectiles} 
    \P [|\hat{\xi}_\nu^m - \xi_\nu(Z)| > \epsilon]  \le 2\exp{\left(- \frac{m\mu_\nu^2 \epsilon^2}{8 L_\nu^2 \sigma^2} \right)},
    \end{equation}
    where $\mu_\nu, L_\nu$ are given in \eqref{eq:expectile-const}. 
\end{theorem}
\begin{proof}
See Appendix~\ref{proofs: upper_bound_mcp_expectile}.
\end{proof}
\begin{remark}
For all $t=0, \dots, T-1$, if $\mathbb{E}[C_t^2] \le K < \infty$, it follows that $Z$ has a finite second moment; furthermore, if $C_t$ is sub-Gaussian, then $Z$ is also sub-Gaussian.
\end{remark}

\begin{remark}
It is easy to see that \Cref{theorem: Upper_bound_for_MCP_expectiles} establishes a uniform upper bound over $\mathcal H_K$, i.e., 
\[
\sup_{(M,C) \in \mathcal H_K} \mathbb{E}\left[ \left|\widehat{\xi}_\nu^{\,m}-\xi_\nu(M,C)\right| \right] \le \frac{L_\nu\sqrt{K}}{\mu_\nu\sqrt{\lfloor n/T\rfloor}}.
\]
Juxtaposing the lower and upper bounds, we obtain
\[
\frac{3}{32\sqrt{n}} \le \inf_\A \sup_{(M,C) \in \mathcal H_K} \mathbb{E}\left[ \left|\widehat{\xi}_{\nu,n}^{\,\A}-\xi_\nu(M,C)\right| \right] \le \frac{L_\nu\sqrt{K}}{\mu_\nu\sqrt{\lfloor n/T\rfloor}},
\]
where $\widehat{\xi}_{\nu,n}^{\,A}$ represents the output of any arbitrary algorithm $\A$ utilizing at most $n$ transitions. For fixed parameters $T$, $\nu$, and $K$, these bounds establish the minimax rate-optimality of the Monte Carlo estimator with respect to the sample size $n$, although they do not necessarily imply optimal dependence on $T, \nu,$ and $K$.

\end{remark}
\begin{remark}
To facilitate a uniform high-probability comparison, fix $\sigma \ge 1$ and define the sub-Gaussian subclass:
\[
\mathcal H_{K,\sigma} := \left\{ (M,C) \in \mathcal H_K : Z \text{ is } \sigma\text{-sub-Gaussian} \right\}.
\]
Because the Gaussian instances constructed in Appendix~\ref{proofs: lower_bound_mcp_expectile} are $1$-sub-Gaussian, they reside directly within $\mathcal H_{K,\sigma}$. Consequently, the lower-bound construction applies over $\mathcal H_{K,\sigma}$, while the concentration bound in Theorem~\ref{theorem: Upper_bound_for_MCP_expectiles} holds uniformly over this class with parameter $\sigma$.

Inverting the tail bound in \eqref{eq: Concentration_upper_bound_for_MCP_expectiles} yields the following high-confidence guarantee: for any $\delta \in \left[\frac{1}{2e}, \frac{1}{2}\right)$,
\[
\left|\widehat{\xi}_\nu^m - \xi_\nu(Z)\right| \le \sqrt{\frac{8L_\nu^2 \sigma^2}{m\mu_\nu^2} \ln\left(\frac{2}{\delta}\right)}
\]
holds with probability at least $1-\delta$. 

For fixed $T$, $\gamma$, $\nu$, $K$, $\sigma$, and
$\delta\in(1/(2e),1/2)$, a direct comparison with the lower bound in 
in \eqref{eq:probability_minimax_bound_mcp_expectile}, whose construction also belongs to
$\mathcal{H}_{K,\sigma}$, confirms that both bounds exhibit identical sample complexity scaling with respect to $n$, i.e., $n^{-1/2}$. Although a mild discrepancy exists regarding the confidence level $\delta$---where the lower bound scales with $\log(1/\delta)$ whereas the inverted upper bound scales as $\sqrt{\log(1/\delta)}$---the matching dependence on $m$ establishes that the $\mathcal{O}(1/\sqrt{m})$ rate guaranteed by Theorem~\ref{theorem: Upper_bound_for_MCP_expectiles} is minimax rate-optimal up to constant factors over $\mathcal H_{K,\sigma}$.
\end{remark}

\section{Policy optimization for expectiles}\label{sec:policy_optimization}

Our primary objective is to identify the policy parameters $\theta^*_{\text{exp}} \in \Theta$ that minimize the expectile risk associated with the return distribution that is induced by the policy $\pi_{\theta}$. Specifically, for a given asymmetry coefficient $\nu \in (0, 1)$, we define the expectile risk $\xi_\nu : \Theta \to \mathbb{R}$ as follows:
    \begin{equation}
        \xi_\nu(\theta) = \argmin_{k} \{\Exp_\theta \left [e_{\nu}(F - k)\right ]\}.
    \end{equation}



\subsection{Policy gradient theorem for expectiles }
For deriving the policy gradient theorem analogue for expectile and to establish the smoothness of the expectile objective, we make the following assumptions. 

\begin{assumption}\label{asm:policy_regularity}
For every state $s \in \mathcal{S}$ and action $a \in \mathcal{A}$, the logarithm of the parameterized policy $\log \pi_\theta(a|s)$ is twice continuously differentiable with respect to the policy parameter $\theta \in \Theta \subseteq \mathbb{R}^d$. 
\end{assumption}

\begin{assumption}\label{asm:boundedness}
There exists a constant $F_{\max} < \infty$ such that $|c(\tau)| \le F_{\max}$ for every $\tau \in \mathcal{T}$.
\end{assumption}



\begin{assumption}
\label{asm:bounded_scores}
For every $\theta \in \Theta$, and every $\tau \in \mathcal{T}$, we have $\|g(\theta,\tau)\| \leq S$.
\end{assumption}

\begin{assumption}
\label{asm:Lipschitz Score}
For every $\theta_1, \theta_2 \in \Theta$, and every $\tau \in \mathcal{T}$, we have $\|g(\theta_1,\tau) - g(\theta_2,\tau)\|_2 \leq L_G \|\theta_1 - \theta_2\|_2$.
\end{assumption}


\begin{assumption}
\label{asm:bounded_pdf}
For every $\theta\in\Theta$, the cumulative discounted cost $F$
admits a probability density $f_\theta$. There exists a constant
$p_{\max}<\infty$, independent of $\theta$, such that
$f_\theta(y)\le p_{\max}$ for every $\theta\in\Theta$ and
almost every $y\in\mathbb{R}$. 
\end{assumption}

\begin{assumptionp}{\ref{asm:bounded_pdf}}\label{asm:continuity}
For every $\theta\in\Theta$,
$   \mathbb{P}_\theta\bigl(F=\xi_\nu(\theta)\bigr)=0.
$
\end{assumptionp}




The assumptions \ref{asm:policy_regularity}, \ref{asm:boundedness}, \ref{asm:bounded_scores}, and \ref{asm:Lipschitz Score} are used for analyzing the policy gradient algorithm, in particular to derive smoothness properties, MSE bounds, and high-probability bounds. Note that these assumptions are required for both risk-neutral and risk-sensitive settings, not specific to expectiles. On the other hand, assumptions \ref{asm:continuity} and \ref{asm:bounded_pdf} are required for the analysis of risk measures, which are very standard in risk-sensitive RL literature, see \cite{tamar2015optimizing}.  
Note that if one assumes that the random variable F is continuous with the density function not approaching infinity, then the probability of $F=k$ for any $k \in R$ is $0$. So, assumption \ref{asm:bounded_pdf} encloses assumption \ref{asm:continuity}.  

Unlike standard methods where one can take the derivative of an explicit expectation directly, expectiles are defined implicitly. The key novelty of this proof is the direct derivation of the expectile gradient using the fact that expectile is the root of \eqref{eq:Expectile_identification_equation}. The result below provides a policy gradient theorem analog for the expectile objective.
\begin{theorem}\label{theorem:policy-gradient-expectile}
Under \Crefrange{asm:policy_regularity}{asm:bounded_scores} and \ref{asm:continuity}, the expectile gradient is given by 
    \begin{align}\label{eq:policy-grad-expectiles}
        \nabla \xi_{\nu}(\theta) = \frac{\mathbb{E}_\theta \left[ l_\nu(F - \xi_\nu(\theta)) G(\theta)\right]}{\mathbb{E}_\theta [ l'_\nu(F - \xi_\nu(\theta)) ]},
       \textrm{ where }
    l'_\nu(x) = \nu \indic{x > 0} + (1 - \nu) \indic{x \le 0 }.
    \end{align}
\end{theorem} 
\begin{proof}
See Appendix~\ref{proofs:policy-gradient-expectile}.
\end{proof}
The result above provides a policy gradient theorem analogue for the expectile objective. The proof is given in \Cref{proofs:policy-gradient-expectile}. 

\begin{lemma} \label{lemma:expectile smoothness proof}
    Under the \Crefrange{asm:policy_regularity}{asm:bounded_pdf}, for every $\theta_1,\theta_2 \in \Theta$, we have
    \begin{align*}
        \|\nabla \xi_{\nu}(\theta_1) - \nabla \xi_{\nu}(\theta_2)\|_2 \le L_\xi\|\theta_1-\theta_2\|_2,
    \end{align*}
    for positive scalar 
    $L_\xi$ which is defined in \Cref{proofs:expectile smoothness proof}.
\end{lemma} 
\begin{proof}
See Appendix~\ref{proofs:expectile smoothness proof}.
\end{proof}
\begin{remark}
    The result above shows that expectile is a smooth function of the policy parameters, a property we utilize to show convergence of our policy gradient algorithm.
\end{remark}

\subsection{Expectile gradient estimation}
Let $\{\tau_1, \tau_2, \dots, \tau_m\}$ be a batch of $m$ trajectories sampled independently from the distribution $p_\theta$ induced by the policy $\pi_\theta$. For each trajectory $\tau_j$, let $c(\tau_j)$ denote the cumulative discounted cost and let $g_1(\theta, \tau_j)$ denote the score function. We use the shorthand notation $\z$ to denote the $m$ realizations as an $n$-dimensional vector. Formally, $\z \triangleq \langle \tau_1,\tau_2,\ldots,\tau_m \rangle$. Then $\hat{\xi}_{\nu, \theta}^m:\mathcal{T}^m \to \Rel$ and $\widehat{\nabla \xi}_{\nu, \theta}^m:\mathcal{T}^m \to \Rel^d$ are our proposed estimators of $\xi_\nu(\theta)$ and $\nabla \xi_\nu(\theta)$ respectively, given by
\begin{align}
    &\frac{1}{m}\sum_{i=1}^{m} l_\nu\left(c(\tau_i) - \hat{\xi}_{\nu, \theta }^m(\z)\right) = 0, \hspace{2mm}
    \widehat{\nabla \xi}_{\nu, \theta}^m (\z) = \frac{\displaystyle \sum_{j=1}^m {l_\nu\left(c(\tau_j)-\hat{\xi}_{\nu, \theta }^m(\z)\right)g_1(\theta,\tau_j)}}{\displaystyle \sum_{j=1}^m l'_\nu\left(c(\tau_j)-\hat{\xi}_{\nu, \theta }^m(\z)\right)}. \label{eq:expectile-gradient-estimator}
\end{align} 
In the following theorem, we obtain an $O(1/m)$ MSE bound on our proposed gradient estimator. The proof of this result employs the stability of the denominator directly via the risk-asymmetry parameter and leverages the Dvoretzky–Kiefer–Wolfowitz (DKW) inequality to cleanly decouple the intertwined sampling and parameter errors.
\begin{theorem} \label{thm:gradient-estimator-expectile-bound}
Under \Crefrange{asm:policy_regularity}{asm:bounded_pdf}, the MSE of the expectile gradient estimator given by (\ref{eq:expectile-gradient-estimator}) is upper bounded by:
\begin{align} 
\mathbb{E} \left\| \widehat{\nabla \xi}_{\nu, \theta}^m - \nabla \xi_\nu(\theta) \right\|^2 \le \frac{C}{m}, 
\end{align}
where $C=\frac{16S^2F_{\max}^2}{\delta^2}
\left(
  \frac{L_\nu^4}{\mu_\nu^2}
  +\max\{\nu,1-\nu\}^2
\right)+
\frac{
  16\max\{\nu,1-\nu\}^2F_{\max}^2S^2(1-2\nu)^2
}{\delta^4}
\left(
  1+\frac{4p_{\max}^2L_\nu^2F_{\max}^2}{\mu_\nu^2}
\right)$, with $\delta=\min\{\nu, 1-\nu\}$.
\end{theorem}

Note that $C$ is a constant  independent of both $m$ and $\theta$, implying that the
bound above holds uniformly over $\theta\in\Theta$.

\begin{proof}
See Appendix~\ref{proofs:gradient-estimator-expectile-bound}.
\end{proof}
Notice that our estimator $\widehat{\nabla \xi}_{\nu, \theta}^m$ is a biased estimate of $\nabla \xi_{\nu,\theta}$. However, as we show in \Cref{thm:gradient-estimator-expectile-bound}, the MSE of the expectile gradient is bounded above by the classic $O(1/m)$, which implies that our proposed gradient estimator (\ref{eq:expectile-gradient-estimator}) is asymptotically unbiased, i.e., it converges to the expectile gradient as the number of samples $m$ tends to infinity. 


\subsection{Expectile-sensitive policy gradient (ESPG) algorithm } 
\label{sec:espg_method}
 The ESPG algorithm performs the following update iteration $N$ times:  
\begin{equation} 
\label{Eq:espg_update_rule}
\theta_{i+1} =   \Pi_{\Theta}\Big(\theta_i - \eta_i \widehat{\nabla }\xi_{\nu}(\theta_i)\Big), 
\end{equation}
where $\{\eta_i\}$ represents the step size sequence and $\widehat{\nabla}\xi_{\nu}(\theta_{i})$ is the gradient estimate formed using $m_i$ trajectories using the current policy $\pi_{\theta_i}$.  Here $\Pi_{\Theta} (x) = \argmin_{a \in \Theta} \|x-a\|^2$  is a projection operator that projects any point $\theta \in \mathbb{R}^d$ onto a closed and convex set $\Theta$. This projection technique has been widely applied in projected gradient methods in the literature. 
After $N$ steps, the ESPG algorithm returns a solution $\theta_R$, where the index $R$ is drawn uniformly at random from $\{1, 2, \dots, N\}$.   \Cref{alg:risk-aware-policy-gradients} presents the pseudocode. 

\begin{algorithm}[H] 
   \caption{Expectile-sensitive policy gradient (ESPG)}
   \label{alg:risk-aware-policy-gradients}
\begin{algorithmic}
   \STATE {\bfseries Input:} initial point $\theta_1 \in \Theta$, step sizes $\{\eta_i\}$, and number of iterations $N$.
   \FOR{$i=1$ {\bfseries to} $N$}
      \STATE Generate $m_i$ trajectories using policy $\pi_{\theta_i}$ and compute gradient estimate $\widehat{\nabla} \xi_{\nu}(\theta_{i})$ using \eqref{eq:expectile-gradient-estimator}.
      \STATE Update policy parameters $\theta_{i+1} \leftarrow \Pi_{\Theta}\Big(\theta_i - \eta_i \widehat{\nabla} \xi_{\nu}(\theta_i)\Big) $, where $\Pi_{\Theta}$ is a projection operator. 
   \ENDFOR
   \STATE {\bfseries Output:} $\theta_R$, where $R$ is chosen uniformly at random from $\{1,2,\ldots,N\}$. 
\end{algorithmic}
\end{algorithm} 
Recall that we aim to solve  $ \;\min_{\theta \in \Theta} \{\xi_{\nu}(\theta) = \rho_\theta\left[F\right]\}$ using a gradient descent type algorithm. Such an algorithm converges to a first-order stationary point (FOSP) of the objective, i.e., a point $\bar \theta$ at which the gradient vanishes. However, in an RL setting, it is difficult to find such a point directly as the underlying model information is unavailable.

To mitigate this issue, several techniques have been explored in the literature. For an unconstrained non-convex optimization problem, an approximate version of FOSP has been studied in the literature (cf. \cite{ghadimi2013stochastic,prashanth2022risk}) and is defined below.

\begin{definition}
    Let $\epsilon > 0$. An output $\hat{\theta} \in \Rel^d$ of a stochastic algorithm for solving \eqref{eq:objective} is an $\epsilon$-FOSP if $\Exp[\|\nabla \xi_{\nu}  (\hat{\theta})\|^2]\leq \epsilon$.
\end{definition} 
For a non-convex constrained optimization problem, a gradient mapping method has been studied in \cite{liu2018zeroth, ghadimi2016mini, j2016proximal}, which is defined as follows:  
\begin{definition} 
\label{def:non_proj}
    Let $\epsilon > 0$ and $\Pi_{\Theta} (x) = \argmin_{a \in \Theta} \|x-a\|^2$. An output $\hat{\theta}$ of a stochastic algorithm for solving \eqref{eq:objective} is an $\epsilon$-FOSP if $\Exp\left[\left\|\mathcal{P}_\Theta\left(\hat{\theta}, \nabla \xi_{\nu}(\hat{\theta}),\eta \right)\right\|^2\right]\leq \epsilon$, where $\mathcal{P}_\Theta (\theta, f(\theta), \eta) = \frac{1}{\eta}[ \theta - \Pi_{\Theta}(\theta-\eta f(\theta))  ]$. 
\end{definition} 
As discussed in \cite{liu2018zeroth}, the projection technique $\mathcal{P}_\Theta$ is aligned with the gradient $\nabla \xi_{\nu}$ and helps to maintain the descent direction of the iterate within $\Theta$. Note that $\mathcal{P}_\Theta (\theta, \nabla \xi_{\nu}, \eta) = \nabla \xi_{\nu}$ for $\Theta = \mathbb{R}^d$. 

 Using $\mathcal{P}_\Theta$ defined in \Cref{def:non_proj}, the iterate \eqref{Eq:espg_update_rule} can be written as follows: 
\begin{equation} 
\label{Eq:espg_update_rule_non}
\theta_{i+1} = \theta_i - \eta_i \mathcal{P}_{\Theta}\left(\theta_i, \widehat{\nabla} \xi_{\nu}(\theta_i),\eta_i \right).   
\end{equation} 
Thus, for non-asymptotic convergence analysis, we provide a bound on $\Exp[\|\mathcal{P}_\Theta(\hat{\theta}, \nabla \xi_{\nu}(\hat{\theta}),\eta )\|^2]$ instead of $\Exp[\|\nabla \xi_{\nu}  (\hat{\theta})\|^2]$ as our proposed method fall under the non-convex constrained optimization.

\begin{theorem}
\label{theorem:non_asymptotic_bound}
Suppose \Crefrange{asm:policy_regularity}{asm:bounded_pdf}
hold, and let $\Theta \subseteq \mathbb{R}^d$ be a nonempty,
closed, convex set.
Let the iterates $\{\theta_i\}$ be generated by
\Cref{alg:risk-aware-policy-gradients}, starting from
$\theta_1 \in \Theta$, with
\[
    m_i = \lceil \sqrt{N} \rceil,
    \qquad
    \eta_i = \frac{1}{\sqrt{N}},
    \qquad i=1,\ldots,N,
\]
where $N$ satisfies $1/\sqrt{N} \leq 1/(2L_\xi)$.
Let $R$ be drawn uniformly from $\{1,\ldots,N\}$,
independently of the algorithm's randomness. Then,
\[
    \Exp\left[
        \left\|
            \mathcal{P}_\Theta\left(
                \theta_R,\nabla\xi_\nu(\theta_R),\eta_R
            \right)
        \right\|^2
    \right]
    \leq
    \frac{
        8\left[\xi_\nu(\theta_1)
        - \xi_\nu(\theta_{\exp}^*)\right]
    }{\sqrt{N}}
    + \frac{6C}{\sqrt{N}},
\]
where $\theta_{\exp}^*$ is an optimal solution over $\Theta$,
$C$ is the constant in
\Cref{thm:gradient-estimator-expectile-bound}, and $\mathcal{P}_\Theta(\theta,g,\eta)$ is as specified in \Cref{def:non_proj}.
\end{theorem}

\begin{proof}
See Appendix~\ref{proof:non_asymptotic_bound}.
\end{proof} 
\begin{remark}
The above expected norm bound 
translates to an iteration complexity of $N=O(\epsilon^{-2})$
for an $\epsilon$-FOSP as defined in \Cref{def:non_proj}.
With the total number of trajectories required equal to
$N\lceil\sqrt{N}\rceil$, the sample complexity is
$O(\epsilon^{-3})$.
Instead, if the stationarity requirement is expressed as
$
\Exp\left[
\left\|\mathcal{P}_\Theta
(\theta_R,\nabla\xi_\nu(\theta_R),\eta_R)\right\|^2
\right]\leq\epsilon^2,
$
then the iteration and sample complexities are
$O(\epsilon^{-4})$ and $O(\epsilon^{-6})$, respectively.
Note that our guarantees do not, in general, ensure convergence
to a local minimum.
\end{remark}

\section{Experiments}\label{sec:experiments}
In this section, we empirically evaluate the performance of our proposed ESPG algorithm on both synthetic and real environments. We do not apply the projection step to both of our experiments. Additional experimental details can be found in Appendix \ref{sec:add_experimt}. 
\subsection{Experiment for the MDP specified in \Cref{exm:four_state}}\noindent
\begin{minipage}[c]{0.46\textwidth}
    \includegraphics[width=\textwidth]{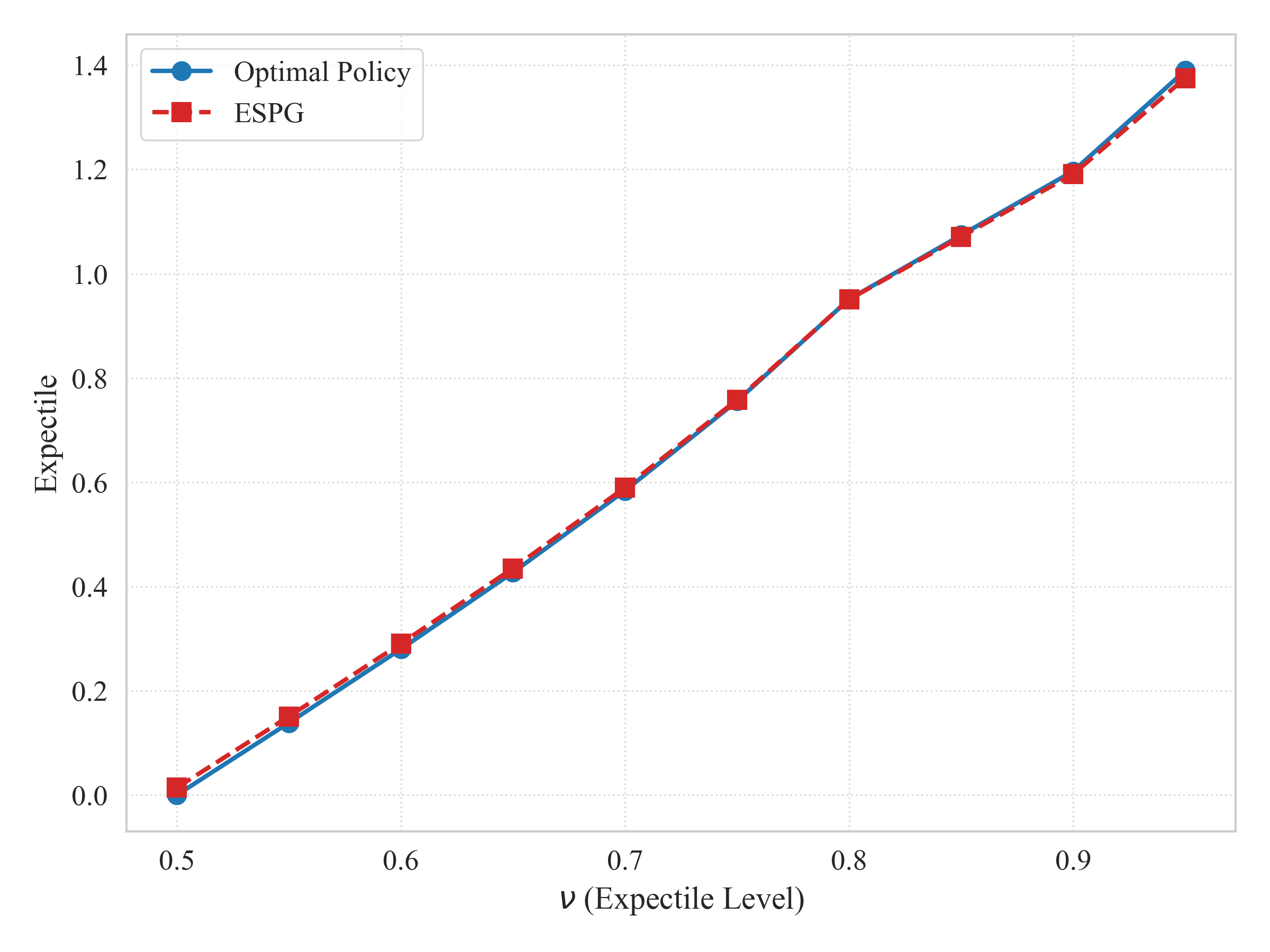}  
	\captionof{figure}{Performance of ESPG on the four-state MDP
specified in \Cref{exm:four_state}, measured by the expectile
of the undiscounted returns.}
    \label{fig:synthetic_four_state} 
\end{minipage} 
\hfill
\begin{minipage}[c]{0.5\textwidth}

We consider a simple four-state MDP as discussed in \Cref{exm:four_state}. We implemented our ESPG algorithm on this synthetic data. In this experiment, we set the discount factor $= 0.9$, the learning rate $=0.05$, the batch size $= 64$, and the number of policy updates $= 50000$ to train our model. We repeat the training runs with five different random seeds. In \Cref{fig:synthetic_four_state}, we show the undiscounted returns of our ESPG algorithm for different expectile values. For each expectile level, the optimal policy is inferred by calculating the expectile for every deterministic policy, as explained in \Cref{exm:four_state}.   It is apparent that our algorithm finds the optimal policy at every expectile level considered.
\end{minipage}



\begin{figure}[h] 
	\centering
	\begin{subfigure}{0.99\textwidth}
		\includegraphics[width=\textwidth]{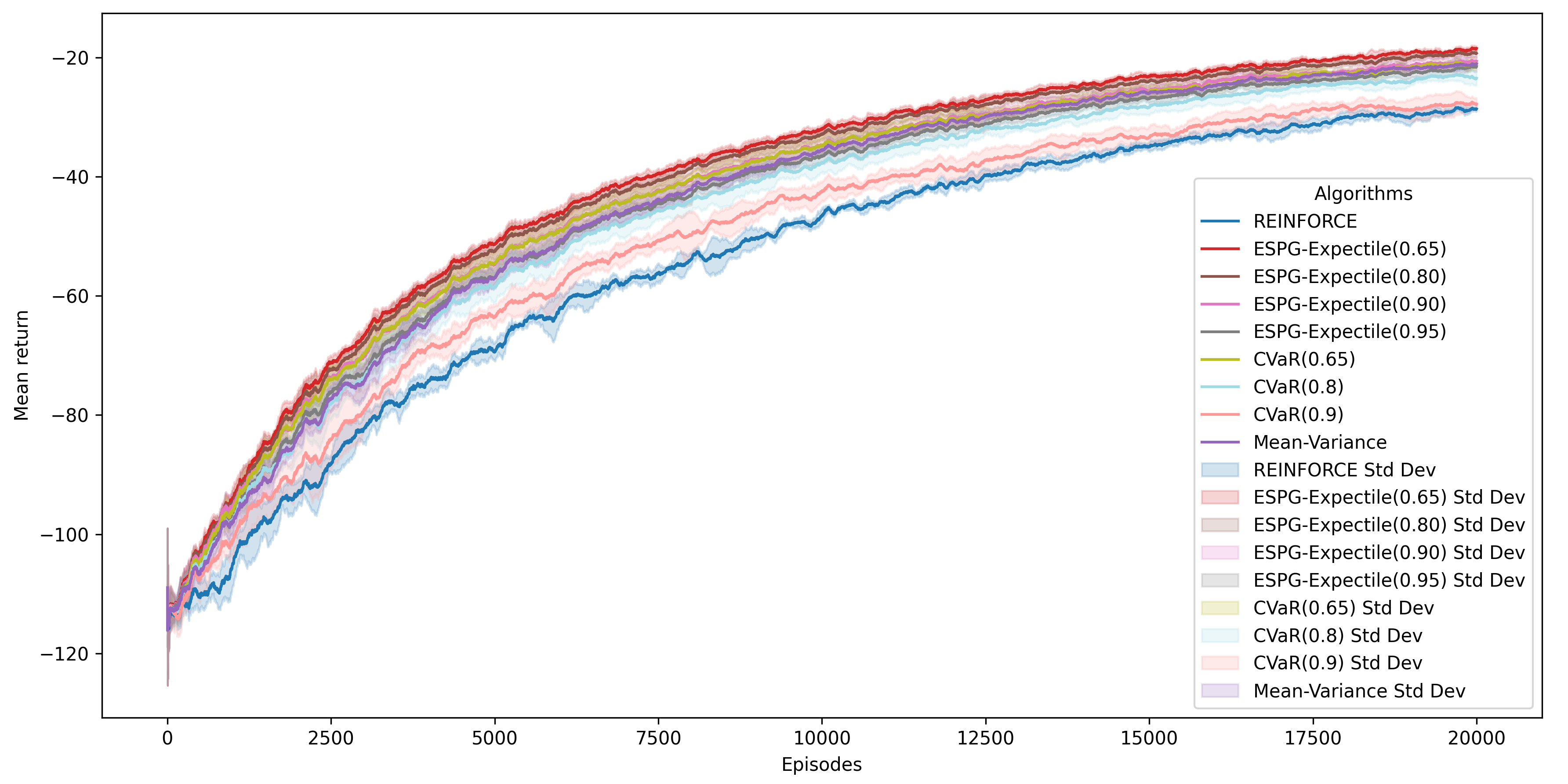}
		\caption{Reacher environment}
        \label{fig:training_recaher}
	\end{subfigure}
	\begin{subfigure}{0.99\textwidth}
		\includegraphics[width=\textwidth]{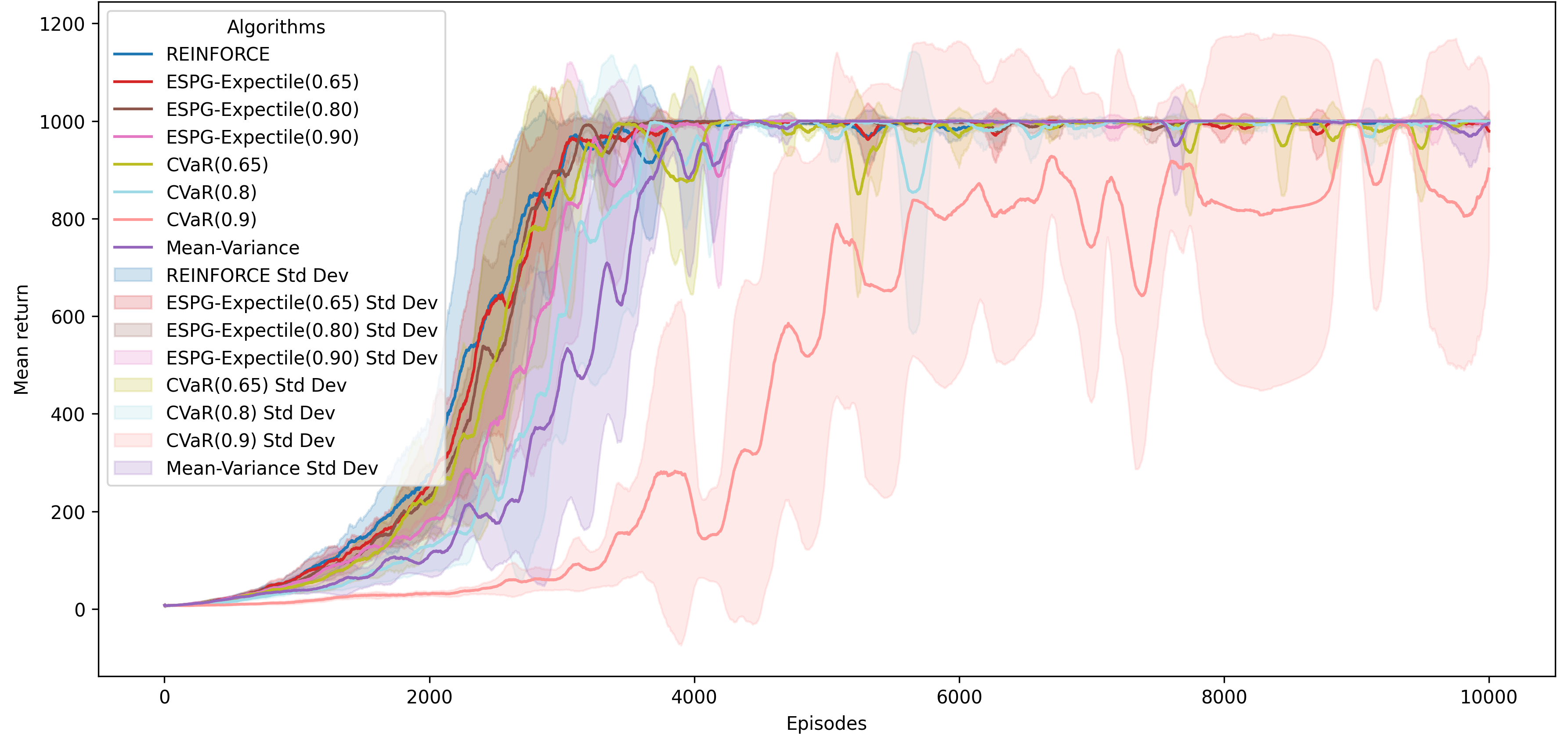}
		\caption{InvertedPendulum environment} 
        \label{fig:training_invtpendulm}
	\end{subfigure} \\
\caption{Training curves on the Reacher and InvertedPendulum
environments, showing the moving average of episodic returns
over the last $100$ episodes.}
\label{fig:training_curve_reacher_invt} 
\end{figure} 
 
\subsection{Experiments on MuJoCo environments}
We perform simulation experiments using MuJoCo environments\footnote{https://gymnasium.farama.org/environments/mujoco/}, namely Reacher-v5, Hopper-v5, and InvertedPendulum-v5. These are popular benchmarks for evaluating RL algorithms. We demonstrate the performance of our ESPG algorithm with $\nu \in \{0.65,0.80,0.90\}$, and compare it with well-known baselines, including REINFORCE, CVaR$_\alpha$ with $\alpha \in \{0.8,0.9\}$, and the mean-variance risk. 


\begin{table}[htbp]
\centering
\caption{Expectile ($\nu = 0.9$) comparison on three MuJoCo environments.}
\label{tab:expectile_results}
\begin{tabular}{|l|c|c|c|}
\toprule
\textbf{Method} & \textbf{Reacher-v5} & \textbf{Hopper-v5} & \textbf{InvertedPendulum-v5} \\
\midrule
REINFORCE                & 31.69          & -242.93          & -981.30 \\
Mean-variance            & 23.23          & -354.22          & -986.08 \\
CVaR$_{0.9}$             & 31.39          & -251.50          & -787.05 \\
ESPG$_{0.9}$ (ours) & \textbf{22.97} & \textbf{-671.76} & \textbf{-996.06} \\
\bottomrule
\end{tabular}
\end{table}

Table \ref{tab:expectile_results} reports the level-$\nu = 0.9$ expectiles of the undiscounted cost distributions for policies trained across three MuJoCo environments. We observe that our ESPG$_{0.9}$ algorithm consistently outperforms all baselines, achieving the lowest estimated cost expectile values across all tasks. In the Reacher and InvertedPendulum environments, ESPG$_{0.9}$ demonstrates a clear advantage over both risk-neutral REINFORCE and risk-aware baselines like Mean-variance and CVaR$_{0.9}$. The performance gap is particularly pronounced in the Hopper environment, where our method achieves a substantially lower cost expectile than REINFORCE and CVaR$_{0.9}$.
Finally, we believe the high variance observed in the Hopper environment is not due to algorithmic instability. Instead, this environment is hard to solve, and the popular baselines such as SAC/DDPG also suffer from high variance. 


\begin{table}[ht]
    \centering
    \caption{Return ($\pm$ standard deviation) comparison on three MuJoCo environments. }
    \label{tab:return_stats}
    \begin{tabular}{|l|r|r|r|} 
        \toprule
        \textbf{Method} & \textbf{Reacher-v5} & \textbf{Hopper-v5}\;\;\;\;\; &\textbf{InvertedPendulum-v5 } \\
        \midrule
        REINFORCE           &$-28.47\pm  3.62$ &$339.99 \pm 129.70$ & $997.87 \pm 39.00$\\
        Mean-variance           & $-20.80 \pm 2.80$ & $430.65 \pm 136.65$ & $998.32 \pm 24.67$\\
        CVaR$_{0.80}$           & $-23.46 \pm 2.89$ & $329.80 \pm 11.80$& $996.83 \pm 41.24$\\
        CVaR$_{0.90}$           & $-28.07 \pm 3.73$ & $350.30 \pm 167.13$ & $965.02 \pm 147.91$\\
        ESPG$_{0.65}$ (\textbf{ours})           & $\mathbf{-18.26 \pm 2.87}$ & $1000.39 \pm 264.48$ & $947.04 \pm 131.93$\\
        ESPG$_{0.80}$ (\textbf{ours})          & $-19.04 \pm 2.83$ & $909.28 \pm 213.14$ & $999.02 \pm 21.60$\\
        ESPG$_{0.90}$ (\textbf{ours})          & $-20.48 \pm 2.83$ & $\mathbf{1055.23 \pm 655.66}$ & $\mathbf{999.56 \pm 11.30}$\\
        \bottomrule
    \end{tabular}
\end{table}  
Table \ref{tab:return_stats} presents the mean and standard deviation of undiscounted returns for the converged policies on three environments for our algorithm as well as the baselines mentioned above. The standard deviation is calculated across test episodes. We observe that our ESPG algorithm outperforms the considered baselines in terms of the highest mean return and among the lowest variance on the Reacher and InvertedPendulum environments, whereas on the Hopper environment, our algorithm achieved three times the mean return of REINFORCE with higher variance. Hopper is an environment where even the standard/popular baselines, such as SAC/DDPG, suffer from high variance. Thus, the variance of ESPG is not due to algorithmic instability; instead, the environment is hard to solve. We observe that REINFORCE did not find a reasonable policy to solve the Hopper problem, whereas our ESPG algorithm found a reasonable policy that achieved a good reward. Hence, the variance comparison is not meaningful in this scenario. All hyperparameters are mentioned in \Cref{tab:hyperparameters} in Appendix \ref{sec:add_experimt}. 
In \Cref{fig:training_curve_reacher_invt}, the X-axis indicates the number of training episodes, while the Y-axis denotes the moving average of episodic undiscounted returns (cumulative rewards) over the last $100$ episodes. We repeated each experiment with five different random seeds. The curves show the mean across these runs, with shaded regions indicating the standard deviation. Additional experimental details are given in
Appendix~\ref{sec:add_experimt}. 

A discussion on performance evaluation with respect to mean return: From \eqref{eq:policy-grad-expectiles} and \eqref{eq:expectile-gradient-estimator}, it follows that the
expectile gradient estimate weights positive and negative cost deviations asymmetrically, whereas the risk-neutral case weights them symmetrically. For instance, for expectile
levels greater than $0.5$, a trajectory with cost above the estimated expectile receives a higher weight in magnitude than one with an equally sized cost deviation below it.
We believe this differential weighting may have improved the mean return of ESPG in the experiments, though we do not have a theoretical justification for this effect.

\section{Conclusions} 
We proposed an estimator for the expectile, established its MSE and high-probability bounds, and showed that our bounds are order-optimal. We presented an expectile-sensitive policy gradient algorithm with a novel gradient estimator. Subsequently, we derived MSE bounds for this gradient estimator. Next, we presented a non-asymptotic bound for our proposed algorithm to find an $\epsilon$-FOSP. Finally, we conducted numerical experiments on a standard RL benchmark to validate our proposed algorithm in practice.

\bibliographystyle{plainnat}
\bibliography{references} 

\appendix

\section{Properties of expectile} \label{properties: expectile_properties} 



\begin{lemma}
    Suppose random variable $X$ has a finite second moment, expectile loss function $\mathcal{L}_\nu(k) = \Exp[e_\nu(X-k)]$ is $L_\nu$-smooth.
\end{lemma} 
\begin{proof}
    To apply the Dominated Convergence Theorem (DCT) for differentiation under the integral sign, we follow Theorem A.5.1 \cite{durrett2019probability} and verify the necessary conditions for $k'$ in a bounded neighborhood $(k-\delta, k+\delta)$ for some $\delta > 0$:
    \begin{enumerate}
        \item $|e_\nu(X - k')| \leq \max(\nu, 1-\nu)(X - k')^2$. Since $X$ has a finite second moment ($\Exp[X^2] < \infty$), we have $\Exp[|e_\nu(X - k')|] < \infty$.
        \item For every $X$, the partial derivative with respect to $k'$ exists everywhere and it is continuous: 
        $\frac{\partial}{\partial k'} e_\nu(X - k') = 2\nu(k' - X)\indic{X > k'} + 2(1 - \nu)(k' - X)\indic{X \leq k'}$.
        \item For all $k' \in (k-\delta, k+\delta)$, the absolute derivative is bounded by:
        \begin{align*}
        \sup_{\theta \in [-\delta, \delta]} \left| \frac{\partial}{\partial k'} e_\nu(X - (k+\theta)) \right| \leq 2\max(\nu, 1-\nu)(|X| + |k| + \delta).  
        \end{align*}
        As X has a finite second moment: 
        \[\Exp \left[ \sup_{\theta \in [-\delta, \delta]} \left| \frac{\partial}{\partial k'} e_\nu(X - (k+\theta)) \right| \right] < \infty. \]
    \end{enumerate}
    Satisfying these conditions, the DCT allows the interchange of the derivative and expectation to yield the first derivative of the expectile loss function:
    \begin{align}
        \frac{d}{dk}\mathcal{L}_\nu(k) &= \frac{d}{dk}  \Big [\Exp\left [e_{\nu}(X - k)\right] \Big ]\nonumber\\
        & = \Exp \left[2\nu(k - X) \indic{X > k} + 2(1 - \nu)(k - X) \indic{X \leq k}\right]. \label{eq:first-derivative-expectile-loss} 
    \end{align}


Let $g(k, X)$ denote the expression inside the expectation in \eqref{eq:first-derivative-expectile-loss}. For any fixed $X$, the slope of $g(k, X)$ with respect to $k$ is either $2\nu$ (when $X > k$) or $2(1-\nu)$ (when $X \le k$). Because this slope is strictly bounded by $L_\nu = 2\max\{\nu, 1-\nu\}$, the function $g(\cdot, X)$ is $L_\nu$-Lipschitz continuous.

Therefore, for any $k_1, k_2 \in \mathbb{R}$:
\begin{align*}
|g(k_1, X) - g(k_2, X)| \le L_\nu |k_1 - k_2|.
\end{align*} 

Taking the expectation over $X$, we obtain:
\begin{align*} 
\left| \frac{d}{dk}\mathcal{L}_\nu(k_1) - \frac{d}{dk}\mathcal{L}_\nu(k_2) \right| 
&= \left| \mathbb{E} [ g(k_1, X) - g(k_2, X) ] \right|  \\
&\le \mathbb{E} \left[ | g(k_1, X) - g(k_2, X) | \right] \\
&\le \mathbb{E} [ L_\nu |k_1 - k_2| ] \\
&= L_\nu |k_1 - k_2|. 
\end{align*} 
\end{proof}

\begin{lemma}
Suppose random variable $X$ has a finite second moment,
the loss function in defining the expectile objective function $\mathcal{L}_\nu(k) $ is strongly convex in $k$. Consequently, the expectile $\xi_\nu$ exists and is unique.
\end{lemma} 

\begin{proof}
Let $g(k, X)$ denote the expression inside the expectation of \eqref{eq:first-derivative-expectile-loss}. For any fixed $X$, the slope of $g(k, X)$ with respect to $k$ is strictly bounded from below by $\mu_\nu = 2\min\{\nu, 1-\nu\}$. Consequently, $g(\cdot, X)$ is strongly monotone.
For any $k_1, k_2 \in \mathbb{R}$, we have:
\begin{align*}
(g(k_1, X) - g(k_2, X))(k_1 - k_2) \ge \mu_\nu (k_1 - k_2)^2.  
\end{align*} 
To prove strong convexity of $\mathcal{L}_\nu(k)$, we evaluate the monotonicity of its first derivative by taking the expectation over $X$:
\begin{align*}
\left( \frac{d}{dk}\mathcal{L}_\nu(k_1) - \frac{d}{dk}\mathcal{L}_\nu(k_2) \right) (k_1 - k_2) 
&= \mathbb{E} \left[ (g(k_1, X) - g(k_2, X))(k_1 - k_2) \right] \\
& \ge \mathbb{E} \left[ \mu_\nu (k_1 - k_2)^2 \right] \\
& \ge \mu_\nu (k_1 - k_2)^2.  
\end{align*} 
\end{proof}

\begin{lemma} \label{lemma:expectile-relation-gmm}
     Let $X_1$ and $X_2$ be random variables with cumulative distribution functions $F_1$ and $F_2$, and corresponding $\nu$-expectiles $\xi_1$ and $\xi_2$, where $\xi_1 < \xi_2$. Let $X_M$ be the mixture random variable governed by the distribution $F_M = w F_1 + (1-w)F_2$ for a mixing weight $w \in (0, 1)$. Then, the $\nu$-expectile of the mixture, $\xi_M$, satisfies $\xi_1 < \xi_M < \xi_2$.
\end{lemma}

\begin{proof}
    Define $u_{X, \nu}$ function for a random variable $X$ at a value $h$ as:
    \[u_{X, \nu}(h) = \mathbb{E}[l_\nu(X - h)],\]
    where $l_\nu(x) = \nu x \mathbb{I}\{x > 0\} + (1 - \nu)x\mathbb{I}\{x \leq 0\}$.

    With Eq \eqref{eq:Expectile_identification_equation} we get,
    \[u_{X_1}(\xi_1) = 0 \text{ and } u_{X_2}(\xi_2) = 0.\]
    
    Because $l_\nu(x)$ is strictly increasing in $x$, $u_{X, \nu}(h)$ is strictly decreasing in $h$. 
    
    Given $\xi_1 < \xi_2$, this monotonicity implies as 
    \[u_{X_2, \nu}(\xi_1) > 0 \text{ and } u_{X_1, \nu}(\xi_2) < 0.\]
    
    By linearity of expectation, the mixture function is given as
    \[u_{X_M, \nu}(h) = w \cdot u_{X_1, \nu}(h) + (1 - w) \cdot u_{X_2, \nu}(h).\]
    
    Evaluate the mixture function at  $\xi_1$ and $\xi_2$ to get 
    \[u_{X_M, \nu}(\xi_1) = w\cdot0 + (1 - w) \cdot u_{X_2, \nu}(\xi_1) > 0 
    \text{ and } u_{X_M, \nu}(\xi_2) = w \cdot u_{X_1, \nu}(\xi_2) + (1 - w) \cdot 0 < 0.\]
    
    Since $u_{X_M, \nu}(h)$ is strictly decreasing and continuous, its unique root $\xi_M$ (where $u_{X_M, \nu}(\xi_M) = 0$) must lie strictly between the bounds where the function changes sign.Therefore, $\xi_1 < \xi_M < \xi_2$. 
    
\end{proof}

\section{Proof of Theorem \ref{thm: mse_bound_expectile}} \label{proofs: mse_bound_expectile} 
\begin{proof}
We prove the first claim, which is a bound on the MSE. 
We follow the technique in \cite{ghosh2024concentration}.
    Using the definitions of \eqref{eq:Expectile_identification_equation} and \eqref{eq:empirical_expectile_identification_equation}, we have
    \begin{equation}
        \frac{1}{m} \sum_{i=1}^{m} \left( l_\nu(X_i - \xi_\nu) - l_\nu(X_i - \hat{\xi}_{\nu}^m) \right) = \frac{1}{m} \sum_{i=1}^{m} \left( l_\nu(X_i - \xi_\nu) - \Exp [l_\nu(X_i - \xi_\nu)] \right).
    \end{equation} 
    
    \noindent \textbf{Case I:} $\hat{\xi}_{\nu}^m \ge \xi_\nu$.
    The monotonicity of $l_{\nu}(x)= \nu x\indic{x > 0} + (1 - \nu) x\indic{x \le 0 }$ implies the following:
    \begin{align*}
        l_\nu(X_i - \xi_\nu) - l_\nu(X_i - \hat{\xi}_{\nu}^m) \ge \frac{\mu_\nu}{2} (\hat{\xi}_{\nu}^m - \xi_\nu).
    \end{align*}
    Summing over $m$ and simplifying leads to
    \begin{align*}
        \frac{1}{m} \sum_{i=1}^{m} \left( l_\nu(X_i - \xi_\nu) - l_\nu(X_i - \hat{\xi}_{\nu}^m) \right) \ge \frac{\mu_\nu}{2} (\hat{\xi}_{\nu}^m - \xi_\nu).
    \end{align*}
    
    \noindent \textbf{Case II:} $\hat{\xi}_{\nu}^m < \xi_\nu$.
    Similarly, we obtain
    \begin{align*}
        \frac{1}{m} \sum_{i=1}^{m} \left( l_\nu(X_i - \hat{\xi}_{\nu}^m) - l_\nu(X_i - \xi_\nu) \right) \ge \frac{\mu_\nu}{2} (\xi_\nu - \hat{\xi}_{\nu}^m).
    \end{align*}
    
    Combining both cases, we can bound the absolute difference as follows:
    \begin{align}
        \label{eq:absolute_difference_upper_bound}
        \left| \frac{2}{m\mu_\nu} \sum_{i=1}^{m} \left( l_\nu(X_i - \xi_\nu) - \Exp [l_\nu(X_i - \xi_\nu)] \right) \right| \ge |\xi_\nu - \hat{\xi}_{\nu}^m|.
    \end{align}
    
    Define $Y_i = l_\nu(X_i - \xi_\nu) - \Exp [l_\nu(X_i - \xi_\nu)]$. Observe that, $\Exp[Y_i] = 0$. Squaring on both sides of the inequality above, and taking the expectations, we obtain
    \begin{align*}
        \Exp[(\xi_\nu - \hat{\xi}_{\nu}^m)^2] &\le \frac{4}{m^2\mu_\nu^2} \Exp \left[ \left( \sum_{i=1}^{m} Y_i \right)^2 \right] 
        \le \frac{4}{m^2\mu_\nu^2} \Exp \left[ \sum_{i=1}^{m} Y_i^2 + \sum_{i,j=1, i \neq j}^{m} Y_i Y_j \right].
    \end{align*}
    Since $X_i$ are i.i.d. and $Y_i$ are zero mean, the cross terms $\Exp[Y_i Y_j]$ for $i \neq j$ are zero. Thus,
    \begin{equation}
        \Exp[(\xi_\nu - \hat{\xi}_{\nu}^m)^2] \le \frac{4}{m^2\mu_\nu^2} \sum_{i=1}^{m} \Exp[Y_i^2] = \frac{4}{m\mu_\nu^2} \Exp[Y_1^2].
    \end{equation}
    Using the Lipschitz property of $l_\nu$, we have $|l_\nu(x - \xi_\nu) - 0| \le \frac{L_\nu}{2}|x - \xi_\nu|$. Therefore,
    \begin{equation}
        \Exp[Y_1^2] = \Exp[(l_\nu(X - \xi_\nu))^2] \le \frac{L_\nu^2}{4} \Exp[(X - \xi_\nu)^2].
    \end{equation}
    Substituting this back gives the final bound:
    \begin{equation}
        \Exp[(\xi_\nu - \hat{\xi}_{\nu}^m)^2] \le \frac{L_\nu^2}{m \mu_\nu^2} \Exp[(X-\xi_\nu)^2].
    \end{equation}
We now turn our attention to proving the high probability bound in the theorem statement.

For proving this claim, we require the following result\footnote{The reader is referred to Lemma B.1 in \cite{ghosh2024concentration} for the proof.}:
\begin{lemma}
    \label{lemma:Lipschitz_subgaussian}
    Let $X \in \R$ be a sub-Gaussian with parameter $\sigma$. If $f:\R \rightarrow \R$ is $L$-Lipschitz, $f(X)$ is sub-Gaussian with parameter $2L\sigma$.
\end{lemma}  
From \eqref{eq:absolute_difference_upper_bound}, we have
\begin{equation}
        |\xi_\nu - \hat{\xi}_{\nu}^m| \le \left| \frac{2}{m\mu_\nu} \sum_{i=1}^{m} \left( l_\nu(X_i - \xi_\nu) - \Exp [l_\nu(X_i - \xi_\nu)] \right) \right|.
\end{equation}
Using \Cref{lemma:Lipschitz_subgaussian}, if we treat the random variable $X$ as $(X_i - \xi_\nu)$ and the Lipschitz function $f$ as $l_\nu$, we establish that the resulting term $l_\nu(X_i - \xi_\nu)$ is sub-Gaussian; consequently, the centered variable $Y_i$ is $L_\nu^2\sigma^2$ sub-Gaussian.
\begin{align*}
        P(|\xi_{\tau} - \hat{\xi}_{\nu}^m| \ge \varepsilon) &\le P\left( \frac{2}{m\mu_\nu} \left| \sum_{i=1}^m Y_i \right| \ge \varepsilon \right) 
        = P\left( \left| \sum_{i=1}^m Y_i \right| \ge \frac{\varepsilon m \mu_\nu}{2} \right) \\
        &\le 2 \exp \left( - \frac{\varepsilon^2 m^2 \mu_\nu^2}{2 \cdot m \cdot 4 L_\nu^2 \sigma^2} \right) 
        \;\;\; \text{(using Hoeffding’s bound)}\\
        &= 2 \exp \left( - \frac{m \varepsilon^2 \mu_\nu^2}{8 L_\nu^2 \sigma^2} \right).
\end{align*} 
Hence proved.
\end{proof} 

\section{Proof of Theorem \ref{thm:lower_bound_mcp_expectile}} \label{proofs: lower_bound_mcp_expectile}

We require the following result to prove \Cref{thm:lower_bound_mcp_expectile}.
\begin{lemma}\label{lemma:expectile-transition-property-normal}
    For a unit-variance Gaussian variable $X \sim \mathcal{N}(\mu, \sigma^2)$, the $\nu$-expectile is $\xi_\nu(X) = \mu + \kappa_\nu \sigma$, where $\kappa_\nu$ is the $\nu$-expectile of the standard normal distribution.
\end{lemma}
\begin{proof}
    The expectile $\xi_\nu$ is defined as the unique solution to \eqref{eq:Expectile_identification_equation}.
    For the standard normal variable $Z \sim \mathcal{N}(0, 1)$, the expectile $\kappa_\nu$ satisfies $\mathbb{E}[l_\nu(Z - \kappa_\nu)] = 0$.\\
    Using the relation $X = \sigma Z + \mu$, we substitute the candidate solution $\xi = \mu + \sigma \kappa_\nu$ into the defining condition. 
    \begin{align*}
    \mathbb{E}[l_\nu(X - \xi)] 
    &= \mathbb{E}[l_\nu((\sigma Z + \mu) - (\mu + \sigma \kappa_\nu))] \\
    &= \mathbb{E}[l_\nu(\sigma(Z - \kappa_\nu))] \\
    &= \sigma \mathbb{E}[l_\nu(Z - \kappa_\nu)] \\
    &= \sigma \cdot 0 = 0.
    \end{align*}
    The above result implies that $\xi_\nu(X) = \mu + \kappa_\nu \sigma$.
\end{proof}

\begin{proof}\textbf{(\Cref{thm:lower_bound_mcp_expectile})}\ \\
    We prove this theorem by reducing the general MCP problem to a specific hard i.i.d. estimation instance, and then deriving the lower bound for that instance using Le Cam's method.
    \subsection*{Part I: Reduction to i.i.d. estimation}
    To establish the lower bound, we utilize a hard instance construction analogous to the proof of Theorem 3.5 in \cite{thoppe2023risk}.
    
    \paragraph{Singleton construction.} We consider a specific instance $(M_0, C)$ where the state space is a singleton $\mathcal{S} = \{s\}$. Consequently, the transition kernel satisfies $P(s|s)=1$ and the initial distribution is $\upsilon(s)=1$. Since the state trajectory is deterministic ($s_n = s$ for all $n$), the estimator $\hat{\xi}_{\nu,n}$ reduces effectively to a function of the observed costs only:
    \begin{equation}
        \hat{\xi}_{\nu,n} \equiv \hat{\xi}_\nu(C_0, \dots, C_{n-1}).
    \end{equation}
    Because the state never changes, the observed costs $C_0, \dots, C_{n-1}$ are independent and identically distributed (i.i.d.) samples from the single-stage cost distribution $C(s)$.
    
    \paragraph{Distributional mapping.} We consider the specific case where the single-stage cost $C(s)$ follows a Gaussian distribution. Consequently, the $T$-length horizon sum $Z = \sum_{t=0}^{T-1} \gamma^t C_t$ is also a Gaussian random variable. This allows us to map the MCP risk estimation problem directly to the i.i.d. Gaussian hypothesis testing problem.

    \paragraph{Minimax reduction inequality.}
    The minimax risk for the general MCP problem is lower-bounded by the risk of the specific singleton instance given above.
Let
$
\mathcal{P}_K=\left\{C:(M_0,C)\in\mathcal{H}_K\right\}$
denote the class of cost distributions for
which the singleton MCP belongs to $\mathcal{H}_K$.
The reduction is formalized as
\begin{equation}
\begin{aligned}
&\inf_A\sup_{(M,C)\in\mathcal{H}_K}
\mathbb{P}\!\left(
\left|\widehat{\xi}_{\nu,n}-\xi_\nu(M,C)\right|
\ge\epsilon
\right)
\\
&\qquad\ge
\inf_{\widehat{\xi}_\nu}\sup_{C\in\mathcal{P}_K}
\mathbb{P}\!\left(
\left|
\widehat{\xi}_\nu(C_0,\ldots,C_{n-1})-\xi_\nu(Z)
\right|
\ge\epsilon
\right).
\end{aligned}
\end{equation}

    \subsection*{Part II: Derivation via Le Cam's method}
    We derive the bound for the reduced problem by constructing two hypotheses that are difficult to distinguish.

    \noindent \textbf{Step 1: Construction of hypotheses} \\
    Consider two Gaussian distributions $P_{+1}$ and $P_{-1}$ for the single-stage cost $C(s)$, defined as follows:
    \begin{align*}
        P_{+1} &: C(s) \sim \mathcal{N}\left(+\epsilon \frac{1-\gamma}{1-\gamma^T}, \frac{1-\gamma^2}{1-\gamma^{2T}}\right), \\
        P_{-1} &: C(s) \sim \mathcal{N}\left(-\epsilon \frac{1-\gamma}{1-\gamma^T}, \frac{1-\gamma^2}{1-\gamma^{2T}}\right).
    \end{align*}
    Consequently, the cumulative discounted cost $Z = \sum_{t=0}^{T-1} \gamma^t C_t$ under each hypothesis follows a unit-variance Gaussian distribution, i.e., $Z \sim \mathcal{N}(+\epsilon, 1)$ under
        $P_{+1}$ and 
        $Z \sim \mathcal{N}(-\epsilon, 1)$ under $P_{-1}$.
    Using \Cref{lemma:expectile-transition-property-normal}, the $\nu$-expectile of the cumulative cost $Z$ under  $P_{+1}$ is $\xi_\nu^{(+1)}(Z) = \epsilon + \kappa_\nu$, and under $P_{-1}$ is $\xi_\nu^{(-1)}(Z) = -\epsilon + \kappa_\nu$.
    Here $\kappa_\nu$ is the $\nu$-expectile of the $\mathcal{N}(0, 1)$ distribution.

Under $P_{+1}$, we have \[
\begin{aligned}
\mathbb{E}_{+1}\!\left[(Z-\xi_\nu(Z))^2\right]
&=
\operatorname{Var}_{+1}(Z)
+\left(\mathbb{E}_{+1}[Z]-\xi_\nu(Z)\right)^2
=1+\left(\epsilon-(\epsilon+\kappa_\nu)\right)^2
\\
&=1+\kappa_\nu^2
\le K.
\end{aligned}
\]
Similarly, under $P_{-1}$, $
\mathbb{E}_{-1}\!\left[(Z-\xi_\nu(Z))^2\right]
=1+\kappa_\nu^2\le K$.
Thus, the singleton MCPs corresponding to $P_{+1}$ and
$P_{-1}$ both belong to $\mathcal{H}_K$.
    
    The separation gap is $|\xi_\nu^{(+1)}(Z) - \xi_\nu^{(-1)}(Z)| = 2\epsilon$. Any estimator $\hat{\xi}_{\nu,n}$ capable of achieving an error less than $\epsilon$ must effectively distinguish between the observed cost sequences generated by these two distributions.

    \noindent \textbf{Step 2: Information distance (KL Divergence)} \\
    In the singleton instance, the estimator observes $n$ i.i.d. costs $C_0, \dots, C_{n-1}$. Let $P_{+1}^n$ and $P_{-1}^n$ denote the joint distributions of these $n$ samples. The Kullback-Leibler (KL) divergence between the single-step cost distributions is:
    \begin{align*}
        D_{KL}(P_{+1} \| P_{-1}) &= \frac{(\mu_1 - \mu_2)^2}{2\sigma^2} 
        = \frac{\left(2\epsilon \frac{1-\gamma}{1-\gamma^T}\right)^2}{2 \left( \frac{1-\gamma^2}{1-\gamma^{2T}} \right)} 
        = 2\epsilon^2 \frac{(1-\gamma)^2}{(1-\gamma^T)^2} \frac{1-\gamma^{2T}}{1-\gamma^2} \\
        &= 2\epsilon^2 \frac{(1-\gamma)(1-\gamma)}{(1-\gamma^T)(1-\gamma^T)} \frac{(1-\gamma^T)(1+\gamma^T)}{(1-\gamma)(1+\gamma)} \\
        &= 2\epsilon^2 \left( \frac{(1-\gamma)(1+\gamma^T)}{(1-\gamma^T)(1+\gamma)} \right).
    \end{align*}
    Since $\gamma \in [0,1)$ and $T \ge 1$, $\left( \frac{(1-\gamma)(1+\gamma^T)}{(1-\gamma^T)(1+\gamma)} \right) \le 1$ holds true, which implies $D_{KL}(P_{+1} \| P_{-1}) \le 2\epsilon^2$. 
    
    For the full history of $n$ independent observations, the divergence is additive:
    \begin{equation}
        D_{KL}(P_{+1}^n \| P_{-1}^n) = n \cdot D_{KL}(P_{+1} \| P_{-1}) \le 2n\epsilon^2. \nonumber
    \end{equation}

    \noindent \textbf{Step 3: Le Cam's lower bound} \\
    Using Le Cam's method, the minimax error probability is bounded by the Total Variation (TV) distance between the joint distributions of the histories. We use Pinsker's inequality in the form $\|P-Q\|_{TV} \le \sqrt{\frac{1}{2} D_{KL}(P\|Q)}$:
    \begin{equation}
        \|P_{+1}^n - P_{-1}^n\|_{TV} \le \sqrt{\frac{1}{2} D_{KL}(P_{+1}^n \| P_{-1}^n)} \le \sqrt{\frac{1}{2} (2n\epsilon^2)} = \sqrt{n\epsilon^2}. \nonumber
    \end{equation}
    The minimax risk is lower bounded by the testing error:
    \begin{equation}
        \label{eq:le_cam_intermediate}
        \inf_{\hat{\xi}_\nu} \sup_{f \in \mathcal{P}_K} \mathbb{P}(|\hat{\xi}_{\nu, n} - \xi_\nu(Z)| \ge \epsilon) \ge \frac{1}{2} \left( 1 - \|P_{+1}^n - P_{-1}^n\|_{TV} \right) \ge \frac{1}{2} \left( 1 - \sqrt{n\epsilon^2} \right),
    \end{equation}
    where the first inequality follows Eqs. (8.2.1) and (8.3.1) of \cite{duchi_ee377_notes}.

    \noindent \textbf{Step 4: Deriving the exponential form} \\
    We utilize the inequality $1 - x \ge \exp\left( - \frac{x}{1-x} \right)$, which holds for $x \in [0, 1)$. Letting $x = \sqrt{n\epsilon^2}$:
    \begin{equation}
        \label{eq:exp_step_1}
        \frac{1}{2} \left( 1 - \sqrt{n\epsilon^2} \right) \ge \frac{1}{2} \exp \left( - \frac{\sqrt{n\epsilon^2}}{1 - \sqrt{n\epsilon^2}} \right). \nonumber
    \end{equation}
    We assume $\epsilon$ is sufficiently small such that $\sqrt{n\epsilon^2} \le \frac{1}{2}$. Under this condition, the denominator $1 - \sqrt{n\epsilon^2} \ge \frac{1}{2}$.
    This implies:
    \begin{equation}
        - \frac{\sqrt{n\epsilon^2}}{1 - \sqrt{n\epsilon^2}} \ge - \frac{\sqrt{n\epsilon^2}}{1/2} = -2\sqrt{n\epsilon^2}. \nonumber
    \end{equation}
    Substituting this back into \eqref{eq:exp_step_1}, we obtain the probability
    \begin{equation}
        \mathbb{P}\left( |\hat{\xi}_{\nu,n} - \xi_\nu(Z)| \ge \epsilon \right) \ge \frac{1}{2} \exp \left(-2\sqrt{n\epsilon^2}\right).
    \end{equation}

    By substituting $\epsilon = \frac{ln{[1/(2\delta)]}}{2\sqrt{n}}$, we have
    \begin{align} \label{eq:substituted_probability_minimax_bound_mcp_expectile}
        \inf_{\A} \sup_{\substack{(M,C)\\\in \mathcal{H}_K}} \mathbb{P}\left( |\hat{\xi}_{\nu,n} - \xi_\nu(M,C)| \ge \frac{\ln{[1/(2\delta)]}}{2\sqrt{n}} \right) \ge \delta.
    \end{align}

    \noindent \textbf{Step 5: Deriving the expectation bound} \\
    To bound the expected error, we use the Markov inequality relationship $\Exp[Z] \ge \delta \mathbb{P}(Z \ge \delta)$.
    Using the intermediate bound from \eqref{eq:le_cam_intermediate} with $\delta = \epsilon$:
    \begin{equation}
        \Exp[|\hat{\xi}_{\nu,n} - \xi_\nu(Z)|] \ge \frac{\epsilon}{2} \left( 1 - \sqrt{n\epsilon^2} \right). \nonumber
    \end{equation}
    We choose $\epsilon = \frac{1}{4\sqrt{n}}$. Substituting this value into the term $\sqrt{n\epsilon^2}$:
    \begin{equation}
        \sqrt{n\epsilon^2} = \sqrt{n \left( \frac{1}{16n} \right)} = \sqrt{\frac{1}{16}} = \frac{1}{4}.  \nonumber
    \end{equation}
    Substituting this back into the expectation inequality:
    \begin{equation}
        \Exp[|\hat{\xi}_{\nu,n} - \xi_\nu(Z)|] \ge \frac{1}{2(4\sqrt{n})} \cdot \left( 1 - \frac{1}{4} \right) = \frac{1}{8\sqrt{n}} \cdot \frac{3}{4} = \frac{3}{32\sqrt{n}}.
    \end{equation}
Thus,
\begin{equation}
\inf_A\sup_{(M,C)\in\mathcal{H}_K}
\mathbb{E}\!\left[
\left|\widehat{\xi}_{\nu,n}-\xi_\nu(M,C)\right|
\right]
\ge
\frac{3}{32\sqrt{n}}.
\end{equation}
The main claim follows.
\end{proof}

\section{Proof of Theorem \ref{theorem: Upper_bound_for_MCP_expectiles}}
\label{proofs: upper_bound_mcp_expectile}
\begin{proof}
    The proof is a complete parallel argument to the proof of \Cref{thm: mse_bound_expectile}  and is given here for the sake of completeness.
    Recall $Z = \sum_{t=0}^{T-1} \gamma^t C(s_t)$. We obtain the MAE bound as follows: 
    \begin{align}
        \Exp |\hat\xi_\nu^m - \xi_\nu(Z)| 
        &\le \sqrt{\Exp |\hat\xi_\nu^m - \xi_\nu(Z)|^2} && \text{(by Jensen's inequality)} \\
        &\le \frac{L_\nu}{\sqrt{m} \mu_\nu} \sqrt{\Exp[(Z-\xi_\nu)^2]} && \text{(using \Cref{thm: mse_bound_expectile})} \label{eq: finite_sample_error_upper_bound_mcp_expectile}.
    \end{align}

    We get to the high probability bound as follows:
    \begin{align}
        \P [|\hat\xi_\nu^m - \xi_\nu(Z)| > \epsilon] 
        &\le 2\exp{\left(- \frac{m\mu_\nu^2 \epsilon^2}{8 L_\nu^2 \sigma^2} \right)} && \text{(using \Cref{thm: mse_bound_expectile})}.
    \end{align} 
\end{proof}

\section{Proof of Theorem \ref{theorem:policy-gradient-expectile}} \label{proofs:policy-gradient-expectile} 
\begin{proof}
Recall that the expectile $\xi_\nu(\theta)$ is implicitly defined as the unique solution to 
\begin{equation}
    \mathcal{M}(\theta, \xi) := \Exp_\theta \left[ l_\nu(F - \xi) \right] = 0.
    \label{eq: exp_root_eq}
\end{equation}
We now compute the partial derivatives of $\mathcal{M}$ with respect to $\theta$ and $\xi$.

\textbf{1. Partial derivative w.r.t. $\theta$:}
Let $p_\theta$ be the distribution induced by the policy. 


To apply the Dominated Convergence Theorem (DCT) for differentiation under the integral sign, we follow Theorem A.5.1 of \cite{durrett2019probability} and verify the necessary conditions for the parameter $\theta$ within a local bounded neighborhood $U$ around $\theta$.

 Notice that $|p_\theta(\tau) l_\nu(F - \xi)| \le p_\theta(\tau) \max(\nu, 1-\nu) |F - \xi| \le p_\theta(\tau)(|F| + |\xi|)$. 
    By \Cref{asm:boundedness}, the trajectory cost is bounded such that $|c(\tau)| \le F_{\max}$, which means $|F| \le F_{\max}$. Integrating this over the trajectory space yields:
    \[ \int |p_\theta(\tau) l_\nu(F - \xi)| d\tau \le \int p_\theta(\tau) (F_{\max} + |\xi|) d\tau = F_{\max} + |\xi| < \infty. \]
    
   For every trajectory $\tau$, the partial derivative with respect to $\theta$ exists everywhere and is continuous:
    \[ \nabla_\theta [p_\theta(\tau) l_\nu(F - \xi)] = l_\nu(F - \xi) \nabla_\theta p_\theta(\tau) = l_\nu(F - \xi) p_\theta(\tau) \nabla_\theta \log p_\theta(\tau). \]
    The interchange of gradient and expectation follows from the dominated convergence theorem, as shown in the following arguments. Fix $\theta\in\Theta$ and $r>0$. Consider the neighborhood around $\theta$ given by $\mathcal{N}_{\theta,r}\triangleq \left\{\vartheta \in \Theta \;\middle|\; \|\vartheta-\theta\|\leq r\right\}$. Then by \Cref{asm:bounded_scores} and integrating along the segment joining $\theta$ and $\vartheta$, we have
\begin{equation}
|\log p_\vartheta(\tau)-\log p_\theta(\tau)|
\le S\|\vartheta-\theta\|.
\end{equation}
Consequently, for every $\vartheta\in\mathcal{N}_{\theta,r}$ we have
\begin{equation}
p_\vartheta(\tau)\le e^{Sr}p_\theta(\tau).
\end{equation}

For fixed $k$ and $\tau,l_\nu(c(\tau)-k)$
does not depend on $\vartheta$. Hence, using
$\nabla_\vartheta p_\vartheta(\tau)
=p_\vartheta(\tau)g(\vartheta,\tau)$, we obtain
\begin{equation}
\nabla_\vartheta
\bigl[p_\vartheta(\tau)l_\nu(c(\tau)-k)\bigr]
=
l_\nu(c(\tau)-k)\,p_\vartheta(\tau)\,g(\vartheta,\tau).
\end{equation}
By the definition of $l_\nu$ and Assumption 2,
\begin{equation}
|l_\nu(c(\tau)-k)|
\le \max\{\nu,1-\nu\}|c(\tau)-k|
\le F_{\max}+|k|.
\end{equation}
Combining these bounds with Assumption 3 gives
\begin{equation}
\begin{aligned}
\left\|
\nabla_\vartheta
\bigl[p_\vartheta(\tau)l_\nu(c(\tau)-k)\bigr]
\right\|
&=
|l_\nu(c(\tau)-k)|\,p_\vartheta(\tau)
\,\|g(\vartheta,\tau)\|\\
&\le S(F_{\max}+|k|)p_\vartheta(\tau)\\
&\le S(F_{\max}+|k|)e^{Sr}p_\theta(\tau).
\end{aligned}
\end{equation}
Since the final bound is independent of $\vartheta$,
\begin{equation}
\sup_{\substack{\vartheta\in\mathcal{N}_{\theta,r}}}
\left\|
\nabla_\vartheta
\bigl[p_\vartheta(\tau)l_\nu(c(\tau)-k)\bigr]
\right\|
\le S(F_{\max}+|k|)e^{Sr}p_\theta(\tau).
\end{equation}

Here $S(F_{\max}+|k|)e^{Sr}$ is a finite constant independent
of $\tau$. Moreover, $p_\theta$ is a probability density, so
\begin{equation}
\int_{\mathcal T}p_\theta(\tau)\,d\tau=1.
\end{equation}
Therefore,
\begin{equation}
\int_{\mathcal T}
S(F_{\max}+|k|)e^{Sr}p_\theta(\tau)\,d\tau
=
S(F_{\max}+|k|)e^{Sr}<\infty.
\end{equation}
This provides an integrable dominating function for the
derivative and justifies differentiation under the integral
by dominated convergence.
    

Satisfying these conditions, the DCT allows the interchange of the derivative and the integral.
\begin{align}
    \nabla_\theta \mathcal{M}(\theta, \xi) &= \nabla_\theta \int p_\theta(\tau) l_\nu(F - \xi) d\tau \nonumber \\
    &= \int \nabla_\theta p_\theta(\tau) l_\nu(F - \xi) d\tau \nonumber \\
    &= \int p_\theta(\tau) \nabla_\theta \log p_\theta(\tau) l_\nu(F - \xi) d\tau \nonumber \\
    &= \mathbb{E}_{\theta} \left[ l_\nu(F - \xi) \sum_{t=0}^{T-1} \nabla_\theta \log \pi_\theta(a_t|s_t) \right] \nonumber \\
    &=\mathbb{E}_{\theta} \left[ l_\nu(F - \xi) G(\theta) \right]
    . \label{eq:grad_theta}
\end{align}

\textbf{2. Partial derivative w.r.t. $\xi$:}

Recall from \eqref{eq: exp_root_eq}
\begin{equation}
\mathcal{M}(\vartheta,k) := \mathbb{E}_\vartheta[l_\nu(F-k)] = \int_{\mathcal T} l_\nu(c(\tau)-k)p_\vartheta(\tau)\,d\tau,
\end{equation}
and define
\begin{equation}
a:=\max\{\nu,1-\nu\},
\qquad
b:=\min\{\nu,1-\nu\}>0.
\end{equation}
Fix $\theta\in\Theta$, and let
\begin{equation}
\xi:=\xi_\nu(\theta),
\qquad
\Delta_h:=\xi_\nu(\theta+h)-\xi,
\end{equation}
where $h\to0$ with $\theta+h\in\Theta$.


To bound the trajectory distribution shift, let $\theta_t := \theta + th \in \Theta$ for $t \in [0, 1]$ by convexity. Applying the chain rule and score identity, 
\begin{equation}
\frac{d}{dt}p_{\theta_t}(\tau)
=
\nabla_\theta p_{\theta_t}(\tau)^\top h
=
p_{\theta_t}(\tau)g(\theta_t,\tau)^\top h.
\end{equation}
Therefore, 
\begin{equation}
p_{\theta+h}(\tau)-p_\theta(\tau)
=
\int_0^1
p_{\theta_t}(\tau)g(\theta_t,\tau)^\top h\,dt.
\end{equation}

Taking the absolute value, applying the Cauchy–Schwarz inequality, and integrating over $\tau$, we get

\begin{equation}
\begin{aligned}
\int_{\mathcal T} |p_{\theta+h}(\tau) - p_\theta(\tau)| \, d\tau 
&\le \int_0^1 \int_{\mathcal T} p_{\theta_t}(\tau) \|g(\theta_t,\tau)\| \, \|h\| \, d\tau \, dt && \text{(Tonelli's theorem)} \\
&\le S\|h\| \int_0^1 \int_{\mathcal T} p_{\theta_t}(\tau) \, d\tau \, dt && \text{(\Cref{asm:bounded_scores})} \\
&= S\|h\|. && \text{($p_{\theta_t}$ integrates to $1$)}
\end{aligned}
\label{eq:exp-pg-l1}
\end{equation}


Next, we bound the shift in the expectile. Since the slopes of $l_\nu$ are bounded below by $b$, for any $u < v$ and all $\tau$, we have
\begin{equation}
l_\nu(c(\tau)-u)-l_\nu(c(\tau)-v)
\ge b(v-u).
\end{equation}

Set $u:=\min\{\xi,\xi+\Delta_h\}$ and
$v:=\max\{\xi,\xi+\Delta_h\}$, ensuring
$v-u=|\Delta_h|$. Taking expectation under $p_{\theta+h}$
gives
\begin{equation}
\begin{aligned}
\mathcal{M}(\theta+h,u)-\mathcal{M}(\theta+h,v)
&=
\mathbb{E}_{\theta+h}\!\left[
l_\nu(F_{\theta+h}-u)-l_\nu(F_{\theta+h}-v)
\right]\\
&\ge b(v-u)
=b|\Delta_h|.
\end{aligned}
\end{equation}
Since $\mathcal{M}(\theta+h,\cdot)$ is monotonically decreasing, 
$\mathcal{M}(\theta+h, u) \ge \mathcal{M}(\theta+h, v)$.
Therefore,
\begin{equation}
b|\Delta_h|
\le
|\mathcal{M}(\theta+h,\xi)-\mathcal{M}(\theta+h,\xi+\Delta_h)|.
\end{equation}


The root equations
$\mathcal{M}(\theta+h,\xi+\Delta_h)=\mathcal{M}(\theta,\xi)=0$ imply
\begin{equation}
\begin{aligned}
|\mathcal{M}(\theta+h,\xi)-\mathcal{M}(\theta+h,\xi+\Delta_h)|
&=
|\mathcal{M}(\theta+h,\xi)-\mathcal{M}(\theta,\xi)|\\
&=
\left|
\int_{\mathcal T}
l_\nu(c(\tau)-\xi)
[p_{\theta+h}(\tau)-p_\theta(\tau)]\,d\tau
\right|.
\end{aligned}
\end{equation}
By \Cref{asm:boundedness}, $|c(\tau)|\le F_{\max}$, which implies the expectile
$\xi \in [-F_{\max},F_{\max}]$. Thus,
\begin{equation}
|l_\nu(c(\tau)-\xi)|
\le a|c(\tau)-\xi|
\le 2aF_{\max}.
\end{equation}
From the foregoing and \eqref{eq:exp-pg-l1}, we have
\begin{equation}
\begin{aligned}
b|\Delta_h|
&\le
2aF_{\max}
\int_{\mathcal T}
|p_{\theta+h}(\tau)-p_\theta(\tau)|\,d\tau\\
&\le 2aF_{\max}S\|h\|.
\end{aligned}
\label{eq:exp-pg-delta}
\end{equation}
Hence, $\Delta_h = \mathcal{O}(\Vert{}h\Vert{})$, and $\Delta_h \to 0$ as $h \to 0$.


To differentiate $\mathcal{M}(\theta, k)$ with respect to $k$ at $k=\xi$, let $s \neq 0$ and define

\begin{equation}
Q_s
:=
\frac{
l_\nu(F-\xi-s)-l_\nu(F-\xi)
}{s}.
\end{equation}
Since $l_\nu$ is $a$-Lipschitz, $|Q_s|\le a$, which is an
integrable bound. Moreover, $l_\nu$ is differentiable away
from zero, and \Cref{asm:continuity} gives
$\mathbb{P}_\theta(F_\theta=\xi)=0$. Consequently,
\begin{equation}
Q_s\longrightarrow -l_\nu'(F-\xi)
\quad\text{almost surely as }s\to0.
\end{equation}
Using dominated convergence theorem,
\begin{equation}
\begin{aligned}
\lim_{s\to0}
\frac{\mathcal{M}(\theta,\xi+s)-\mathcal{M}(\theta,\xi)}{s}
&=
\lim_{s\to0}\mathbb{E}_\theta[Q_s]\\
&=
-\mathbb{E}_\theta[l_\nu'(F-\xi)]\\
\end{aligned}
\end{equation}
Since
$\Delta_h\to0$,we have
\begin{equation}
\mathcal{M}(\theta,\xi+\Delta_h)-\mathcal{M}(\theta,\xi)
=
-\mathbb{E}_\theta[l_\nu'(F-\xi)]\Delta_h+o(|\Delta_h|).
\label{eq:exp-pg-k-expansion}
\end{equation}


By differentiability with respect to the first argument and the gradient in \eqref{eq:grad_theta}, expanding $\mathcal{M}(\cdot, \xi)$ about $\theta$ gives
\begin{equation}
\mathcal{M}(\theta+h,\xi)-\mathcal{M}(\theta,\xi) = \mathbb{E}_\theta[l_\nu(F-\xi)G(\theta)]^\top h + o(\|h\|).
\label{eq:exp-pg-theta-expansion}
\end{equation}


It remains to control the mixed increment
\begin{equation}
\begin{aligned}
R_h:={}&
\bigl[\mathcal{M}(\theta+h,\xi+\Delta_h)-\mathcal{M}(\theta+h,\xi)\bigl]-\bigl[\mathcal{M}(\theta,\xi+\Delta_h)-\mathcal{M}(\theta,\xi)\bigl].
\end{aligned}
\label{eq:def_R_h}
\end{equation}
Expressing each term explicitly via the trajectory integral results in
\begin{equation}
\begin{aligned}
R_h
={}&
\int_{\mathcal T}
\bigl[
l_\nu(c(\tau)-\xi-\Delta_h)
-l_\nu(c(\tau)-\xi)
\bigr]\times
[p_{\theta+h}(\tau)-p_\theta(\tau)]\,d\tau.
\end{aligned}
\end{equation}
The $a$-Lipschitz continuity of $l_\nu$ gives
\begin{equation}
\left|
l_\nu(c(\tau)-\xi-\Delta_h)-l_\nu(c(\tau)-\xi)
\right|
\le a|\Delta_h|.
\end{equation}
Thus, by the triangle inequality,
\eqref{eq:exp-pg-l1}, and \eqref{eq:exp-pg-delta},
\begin{equation}
\begin{aligned}
|R_h|
&\le
a|\Delta_h|
\int_{\mathcal T}
|p_{\theta+h}(\tau)-p_\theta(\tau)|\,d\tau\\
&\le aS|\Delta_h|\|h\|\\
&\le
\frac{2a^2F_{\max}S^2}{b}\|h\|^2
=\mathcal{O}(\|h\|^2).
\end{aligned}
\label{eq:exp-pg-mixed}
\end{equation}

Finally, the root equations and \eqref{eq:def_R_h} give
\begin{equation}
\begin{aligned}
0 &= \mathcal{M}(\theta+h,\xi+\Delta_h) - \mathcal{M}(\theta,\xi) \\
&= [\mathcal{M}(\theta+h,\xi) - \mathcal{M}(\theta,\xi)] + [\mathcal{M}(\theta,\xi+\Delta_h) - \mathcal{M}(\theta,\xi)] + R_h && \text{(by Eq. \ref{eq:def_R_h})} \\
&= \mathbb{E}_\theta[l_\nu(F-\xi)G(\theta)]^\top h - \mathbb{E}_\theta[l_\nu'(F-\xi)]\Delta_h + o(\|h\|). && \text{(by Eq. \ref{eq:exp-pg-k-expansion}, \ref{eq:exp-pg-theta-expansion} and \ref{eq:exp-pg-mixed})}
\end{aligned}
\end{equation}
In the above, we used $\Delta_h=\mathcal{O}(\|h\|)$ to
absorb $o(|\Delta_h|)$ into $o(\|h\|)$.

Since $\mathbb{E}_\theta[l_\nu'(F-\xi)] \ge b>0$, we conclude that
\begin{equation}
\Delta_h
=
\frac{\mathbb{E}_\theta[l_\nu(F-\xi)G(\theta)]^\top h}{\mathbb{E}_\theta[l_\nu'(F-\xi)]}+o(\|h\|).
\end{equation}
Thus,
\begin{equation}
\nabla\xi_\nu(\theta)
=
\frac{
\mathbb{E}_\theta[
l_\nu(F-\xi_\nu(\theta))G(\theta)]
}{
\mathbb{E}_\theta[
l_\nu'(F-\xi_\nu(\theta))]
}.
\end{equation}
The claim follows.
\end{proof}

\section{Proof of Lemma \ref{lemma:expectile smoothness proof}} \label{proofs:expectile smoothness proof}
To prove \Cref{lemma:expectile smoothness proof}, we require the following result.
\begin{lemma} \label{lemma: Lipschitz-property-of-expectile}
Under \Crefrange{asm:policy_regularity}{asm:bounded_scores} and \ref{asm:continuity}, the expectile $\xi_\nu(\theta)$ is Lipschitz continuous with respect to the parameter vector $\theta$. Specifically, there exists a constant $L_m$ such that for any $\theta_1, \theta_2 \in \Theta$:
$$ \| \xi_\nu(\theta_1) - \xi_\nu(\theta_2) \|_2 \le L_m \| \theta_1 - \theta_2 \|_2, $$

where $L_m = \frac{2F_{\max} \max(\nu, 1-\nu) S}{\min(\nu, 1-\nu)}$.
\end{lemma}

\begin{proof}

Using \Cref{theorem:policy-gradient-expectile}, the gradient of the expectile is given by:
\begin{equation}
    \nabla_\theta \xi_\nu(\theta) = \frac{\mathbb{E}_\theta[l_\nu(F - \xi_\nu(\theta)) G(\theta)]}{\mathbb{E}_\theta[l'_\nu(F - \xi_\nu(\theta))]}.
\end{equation}

We bound the norm of the expectile gradient as follows:
\begin{align}
    \| \nabla_\theta \xi_\nu(\theta) \|_2
    &= \left\| \frac{\mathbb{E}_\theta[l_\nu(F - \xi_\nu(\theta)) G(\theta)]}{\mathbb{E}_\theta[l'_\nu(F - \xi_\nu(\theta))]} \right\|_2 \nonumber \\
    &\le \frac{\left\| \mathbb{E}_\theta[l_\nu(F - \xi_\nu(\theta)) G(\theta)] \right\|_2}{\min(\nu, 1-\nu)} \nonumber \\
    &\le \frac{\mathbb{E}_\theta [ |l_\nu(F - \xi_\nu(\theta))| \cdot \|G(\theta)\|_2 ]}{\min(\nu, 1-\nu)} 
    && \begin{aligned} 
         &\text{(by Jensen's inequality and} \\ 
         &\text{\quad Cauchy-Schwarz inequality)} 
       \end{aligned} \nonumber \\
    &\le \frac{\mathbb{E}_\theta [ 2F_{\max} \cdot \max(\nu, 1-\nu) \cdot \|G(\theta)\|_2 ]}{\min(\nu, 1-\nu)} && \text{(using \Cref{asm:boundedness})} \nonumber \\
    &\le \frac{\mathbb{E}_\theta [ 2F_{\max} \cdot \max(\nu, 1-\nu) \cdot S ]}{\min(\nu, 1-\nu)} && \text{(using \Cref{asm:bounded_scores})} \nonumber \\
    &= \frac{2F_{\max} \max(\nu, 1-\nu) S}{\min(\nu, 1-\nu)}. \label{eq:expectile-gradient-constant-upper-bound}
\end{align}

With the bound on the gradient norm from \eqref{eq:expectile-gradient-constant-upper-bound}, we apply the Mean Value Theorem to obtain the main result:
$$\| \xi_\nu(\theta_1) - \xi_\nu(\theta_2) \|_2 \le L_m \| \theta_1 - \theta_2 \|_2,$$
where $L_m = \frac{2F_{\max} \max(\nu, 1-\nu) S}{\min(\nu, 1-\nu)}$.
\end{proof}

\begin{proof}(\Cref{lemma:expectile smoothness proof})

    The expectile gradient expression is given by:
    \begin{equation}
        \nabla_\theta \xi_\nu(\theta) = \frac{N(\theta)}{D(\theta)} = \frac{\mathbb{E}_\theta[l_\nu(F - \xi_\nu(\theta)) G(\theta)]}{\mathbb{E}_\theta[l'_\nu(F - \xi_\nu(\theta))]}.
    \end{equation}
    
    $D(\theta)$ can be rewritten using the CDF $\mathcal{F}_{\theta}$ of $F$ as follows:
    \begin{align}
        D(\theta)
        &= \mathbb{E}_\theta[\nu \indic{F > \xi_\nu(\theta)} + (1-\nu )\indic{F \le \xi_\nu(\theta)}] \nonumber \\
        &= \nu \mathbb{P}_\theta(F > \xi_\nu(\theta)) + (1-\nu )\mathbb{P}_\theta(F \le \xi_\nu(\theta)) \nonumber \\
        &= \nu(1-\mathcal{F}_{\theta}(\xi_\nu(\theta))) + (1-\nu)\mathcal{F}_{\theta}(\xi_\nu(\theta)) \nonumber \\
        &= \nu + (1-2\nu)\mathcal{F}_{\theta}(\xi_\nu(\theta)). \label{eq:D_in_CDF_Form_expectile_smoothness_proof}
    \end{align}
    Notice that
    \begin{align}\label{proof:expectile-gradient-rearrange}
    &
        \| \nabla_\theta \xi_\nu(\theta_1) - \nabla_\theta \xi_\nu(\theta_2) \|_2 \\
        &= \left\| \frac{N(\theta_1)}{D(\theta_1)} - \frac{N(\theta_2)}{D(\theta_2)} \right\|_2 \nonumber\\
        &\le \frac{1}{|D(\theta_1)| |D(\theta_2)|} \left( \|N(\theta_1)\|_2 |D(\theta_2) - D(\theta_1)| + |D(\theta_1)| \|N(\theta_1) - N(\theta_2)\|_2 \right).
    \end{align}
    
    We first bound $|D(\theta_2) - D(\theta_1)|$ as follows:
    \begin{align}
        |D(\theta_2) - D(\theta_1)|
        &=|1-2\nu|\left|\mathcal{F}_{\theta_2}(\xi_\nu(\theta_2)) - \mathcal{F}_{\theta_1}(\xi_\nu(\theta_1))\right| \nonumber \\
        &\le |1-2\nu| \left( \left|\mathcal{F}_{\theta_2}(\xi_\nu(\theta_2)) - \mathcal{F}_{\theta_2}(\xi_\nu(\theta_1))\right| + \left|\mathcal{F}_{\theta_2}(\xi_\nu(\theta_1)) - \mathcal{F}_{\theta_1}(\xi_\nu(\theta_1))\right| \right) \nonumber.
    \end{align}
    To bound the second term, $|\mathcal{F}_{\theta_2}(\xi_\nu(\theta_1)) - \mathcal{F}_{\theta_1}(\xi_\nu(\theta_1))|$, we analyze the sensitivity of CDF $\mathcal{F}_\theta$ with respect to the policy parameters $\theta$. We utilize Lemma 8 from \cite{vijayan2021policy}, which derives the gradient of the CDF using the likelihood ratio method, resulting in the following expression:
    \[\nabla_\theta \mathcal{F}_\theta(y) = \mathbb{E}_{\tau}\left[\indic{F \le y} g(\theta, \tau)\right].\] 
    Using \Cref{asm:bounded_scores}, which bounds the norm of the trajectory-level score function by $S$ (i.e., $\|g(\theta, \tau)\| \le S$) and the fact that the indicator function is bounded above by $1$, we have
    $$\|\nabla_\theta \mathcal{F}_\theta(y)\| \le \mathbb{E}[1 \cdot S] = S.$$
    Thus, $|\mathcal{F}_{\theta_2}(y) - \mathcal{F}_{\theta_1}(y)| \le S \|\theta_2 - \theta_1\|_2$.

    Notice that 
    \begin{align}
    |D(\theta_2) - D(\theta_1)|
        &\le |1-2\nu| \left( \left|\mathcal{F}_{\theta_2}(\xi_\nu(\theta_2)) - \mathcal{F}_{\theta_2}(\xi_\nu(\theta_1))\right| + \left|\mathcal{F}_{\theta_2}(\xi_\nu(\theta_1)) - \mathcal{F}_{\theta_1}(\xi_\nu(\theta_1))\right| \right) \nonumber \\
        &\le |1-2\nu| \left( p_{\text{max}} |\xi_\nu(\theta_2) - \xi_\nu(\theta_1)| + S\|\theta_2 - \theta_1\|_2 \right)
        \nonumber \\
        &\le |1-2\nu| \left(p_{\text{max}}L_m+S\right) \|\theta_2 - \theta_1\|_2 \quad \text{(using \Cref{lemma: Lipschitz-property-of-expectile})}.
    \end{align}

    Now, we bound $\|N(\theta_1) - N(\theta_2)\|_2$, which apprears in the numerator of \eqref{proof:expectile-gradient-rearrange}.
    \begin{align}
        \|N(\theta_1) - N(\theta_2)\|_2  
        &\leq \left\|\mathbb{E}_{\theta_1}\left[l_\nu\left(F - \xi_\nu(\theta_1)\right) G(\theta_1) - l_\nu\left(F - \xi_\nu(\theta_2)\right) G(\theta_2)\right]\right\|_2 \nonumber\\
        &+ \left\|\mathbb{E}_{\theta_1}\left[l_\nu\left(F - \xi_\nu(\theta_2)\right) G(\theta_2)\right] - \mathbb{E}_{\theta_2}\left[l_\nu\left(F - \xi_\nu(\theta_2)\right) G(\theta_2)\right]\right\|_2 \nonumber \\
        & = I_1 + I_2.
    \end{align}
    For the first term on the RHS (denoted as $I_1$), we have
    \begin{align}
    I_1 &= \left\|\mathbb{E}_{\theta_1}\left[l_\nu\left(F - \xi_\nu(\theta_1)\right) G(\theta_1) - l_\nu\left(F - \xi_\nu(\theta_2)\right) G(\theta_2)\right]\right\|_2 \nonumber \\
    &\le \mathbb{E}_{\theta_1}\left[\left\|l_\nu\left(F - \xi_\nu(\theta_1)\right) G(\theta_1) - l_\nu\left(F - \xi_\nu(\theta_2)\right) G(\theta_2) \right\|_2\right] \nonumber \\
    &= \mathbb{E}_{\theta_1}\left[\left\| l_\nu\left(F - \xi_\nu(\theta_1)\right) G(\theta_1) - l_\nu\left(F - \xi_\nu(\theta_2)\right) G(\theta_1) \right. \right. \nonumber\\
    &\quad + l_\nu\left(F - \xi_\nu(\theta_2)\right) G(\theta_1) - l_\nu\left(F - \xi_\nu(\theta_2)\right) G(\theta_2) \|_2] \nonumber \\
    &\le \mathbb{E}_{\theta_1}\left[ \|G(\theta_1)\|_2 \left| l_\nu(F - \xi_\nu(\theta_1)) - l_\nu(F - \xi_\nu(\theta_2)) \right| \right] \nonumber \\
    &\quad + \mathbb{E}_{\theta_1}\left[ \left| l_\nu(F - \xi_\nu(\theta_2)) \right| \|G(\theta_1) - G(\theta_2)\|_2 \right] \nonumber \\
    &\le S \cdot \max(\nu,1-\nu) \mathbb{E}_{\theta_1}\left[ \left| (F - \xi_\nu(\theta_1)) - (F - \xi_\nu(\theta_2)) \right| \right] \nonumber \\
    &\quad + 2\max(\nu,1-\nu)F_{\max} \cdot L_G \|\theta_1 - \theta_2\|_2 \nonumber \\
    &\le S \max(\nu,1-\nu) L_m \|\theta_1 - \theta_2\|_2 + 2\max(\nu,1-\nu)F_{\max} L_G \|\theta_1 - \theta_2\|_2 \nonumber \\
    &= \left( S \max(\nu,1-\nu) L_m + 2\max(\nu,1-\nu)F_{\max} L_G \right) \|\theta_1 - \theta_2\|_2.
    \end{align}
    
    We bound the second term ($I_2$) as follows:    
    \begin{align}
        I_2 &= \left\|\mathbb{E}_{\theta_1}\left[l_\nu\left(F - \xi_\nu(\theta_2)\right) G(\theta_2)\right] - \mathbb{E}_{\theta_2}\left[l_\nu\left(F - \xi_\nu(\theta_2)\right) G(\theta_2)\right]\right\|_2 \nonumber \\
        &= \left\| \int l_\nu\left(c(\tau) - \xi_\nu(\theta_2)\right) g(\theta_2, \tau) \left( p_{\theta_1}(\tau) - p_{\theta_2}(\tau) \right) d\tau \right\|_2 \nonumber \\
        &\le \int \left\| l_\nu\left(c(\tau) - \xi_\nu(\theta_2)\right) g(\theta_2, \tau) \right\|_2 \cdot \left| p_{\theta_1}(\tau) - p_{\theta_2}(\tau) \right| d\tau \nonumber \\
        &\le \left( \sup_{\tau} \left\| l_\nu\left(c(\tau) - \xi_\nu(\theta_2)\right) g(\theta_2, \tau) \right\|_2 \right) \cdot \int \left| p_{\theta_1}(\tau) - p_{\theta_2}(\tau) \right| d\tau \nonumber \\
        &\le (2\max(\nu,1-\nu)F_{\max} S) \cdot \int_0^1 \left( \int p_{\theta_t}(\tau) \|g(\theta_t, \tau)\|_2 \, d\tau \right) dt \cdot \|\theta_1 - \theta_2\|_2 \nonumber \\
        & \hspace{25em} (\text{where } \theta_t = \theta_2 + t(\theta_1 - \theta_2)) \nonumber \\
        &= (2\max(\nu,1-\nu)F_{\max} S) \cdot \int_0^1 \mathbb{E}_{\tau \sim p_{\theta_t}} \left[ \|g(\theta_t, \tau)\|_2 \right] dt \cdot \|\theta_1 - \theta_2\|_2 \nonumber \\
        &\le (2\max(\nu,1-\nu)F_{\max} S) \cdot S \|\theta_1 - \theta_2\|_2.
    \end{align}
    Combining $I_1$ and $I_2$, we obtain the final Lipschitz bound for the numerator:
    \begin{align}
        &
        \|N(\theta_1) - N(\theta_2)\|_2 \\
        &\le \big( S \max(\nu,1-\nu) L_m
        + 2\max(\nu,1-\nu)F_{\max} L_G + 2\max(\nu,1-\nu)F_{\max} S^2 \big) \|\theta_1 - \theta_2\|_2 \nonumber \\
        &= L_N \|\theta_1 - \theta_2\|_2. \label{eq:lipschitz_constant}
    \end{align}
    By combining these results, we obtain
    \begin{align}
        \| \nabla_\theta \xi_\nu(\theta_1) - \nabla_\theta \xi_\nu(\theta_2) \|_2 
        & \le L_\xi \|\theta_1-\theta_2\|_2,
    \end{align}
where $L_\xi=
\dfrac{(2\max(\nu,1-\nu)F_{\max}S|1-2\nu| \left(p_{\max}L_m+S\right)+\max(\nu,1-\nu)L_N)}{(\min(\nu, 1-\nu))^2}$.
\end{proof}
\section{Proof of Theorem \ref{thm:gradient-estimator-expectile-bound}} \label{proofs:gradient-estimator-expectile-bound}

\begin{proof}
We analyze the Mean Squared Error (MSE) of the gradient estimator $\widehat{\nabla \xi}_{\nu, \theta}^m = \frac{\hat{N}(\hat{\xi})}{\hat{D}(\hat{\xi})}$ relative to the true gradient $\nabla_\theta \xi_\nu(\theta) = \frac{N(\xi)}{D(\xi)}$, where $\hat{\xi} = \hat{\xi}_{\nu,\theta}^m$ and $\xi = \xi_\nu(\theta)$.
Using the ratio decomposition $\frac{a}{b} - \frac{A}{B} = \frac{a-A}{b} + \frac{A(B-b)}{bB}$ and observing that the denominator is bounded as $\delta \le |D(\cdot)| \le 1-\delta$ where $\delta = \min(\nu, 1-\nu)$:
\begin{equation} \label{ep:expectile-mse-bound-breakdown}
    \mathbb{E} \left\| \widehat{\nabla \xi}_{\nu, \theta}^m - \nabla_\theta \xi_\nu(\theta) \right\|^2 
    \le \frac{2}{\delta^2} \underbrace{\mathbb{E} \|\hat{N}(\hat{\xi}) - N(\xi)\|^2}_{term (A)} 
    + \frac{2 \|N(\xi)\|^2}{\delta^4} \underbrace{\mathbb{E} |\hat{D}(\hat{\xi}) - D(\xi)|^2}_{term (B)}.
\end{equation}

\noindent \textbf{Bounding term (A) in \eqref{ep:expectile-mse-bound-breakdown}} \\
Decomposing into parameter and sampling error:
\[
    \mathbb{E} \|\hat{N}(\hat{\xi}) - N(\xi)\|^2 \le 2\mathbb{E} \|\hat{N}(\hat{\xi}) - \hat{N}(\xi)\|^2 + 2\mathbb{E} \|\hat{N}(\xi) - N(\xi)\|^2
\]
For the parameter error, using $L_\nu$-Lipschitz continuity of $l_\nu$:
\begin{align}
    \mathbb{E} \|\hat{N}(\hat{\xi}) - \hat{N}(\xi)\|^2 
    &= \mathbb{E} \left\| \frac{1}{m} \sum_{i=1}^m g(\theta, \tau_i) (l_\nu(c(\tau_i) - \hat{\xi}) - l_\nu(c(\tau_i) - \xi)) \right\|_2^2 \nonumber\\
    &\le S^2 L_\nu^2 \mathbb{E} |\hat{\xi}-\xi|^2 \nonumber \\
    &\le \frac{S^2 L_\nu^4}{m \mu_\nu^2} \mathbb{E}_\theta[(F -\xi)^2].
\end{align}

Now we will try to get bound on sampling error,
\begin{equation}
    \mathbb{E} |\hat{N}(\xi) - N(\xi)|^2 \le \frac{4 S^2 (\max(\nu,1-\nu))^2 F_{\max}^2}{m}
\end{equation}

    \noindent \textbf{Bounding term (B) in \eqref{ep:expectile-mse-bound-breakdown}} \\
Similarly decomposing the denominator error:
\[
    \mathbb{E} |\hat{D}(\hat{\xi}) - D(\xi)|^2 \le 2\mathbb{E} |\hat{D}(\hat{\xi}) - D(\hat{\xi})|^2 + 2\mathbb{E} |D(\hat{\xi}) - D(\xi)|^2.
\]
The sampling error uses the Dvoretzky–Kiefer–Wolfowitz inequality for the empirical CDF $\hat{\mathcal{F}}_m$:
\[
    \mathbb{E} |\hat{D}(\hat{\xi}) - D(\hat{\xi})|^2 = (2\nu-1)^2 \mathbb{E} |\hat{\mathcal{F}}_m(\hat{\xi}) - \mathcal{F}(\hat{\xi})|^2 \le \frac{(2\nu-1)^2}{m}. 
\]
Now, we bound the
\begin{align}  
    \mathbb{E} |D(\hat{\xi}) - D(\xi)|^2
    &= \mathbb{E} |(1-2\nu)(\mathcal{F}(\xi)-\mathcal{F}(\hat{\xi}))|^2 \nonumber \\
    &\le \mathbb{E} |(1-2\nu)p_{\max} |\xi-\hat\xi||^2 \nonumber \\
    &= (1-2\nu)^2 p_{\max}^2 \mathbb{E}|\xi - \hat\xi|^2 \nonumber \\
    &\le \frac{(1-2\nu)^2 p_{\max}^2 L_\nu^2}{m \mu_\nu^2}  \mathbb{E}_\theta[(F -\xi)^2].   
\end{align}

Now, by combining all the above results, we get

\begin{align}
&
    \mathbb{E} \left\| \widehat{\nabla \xi}_{\nu, \theta}^m - \nabla_\theta \xi_\nu(\theta) \right\|^2 \\ 
    & \le \frac{4}{\delta^2} \left( \frac{S^2 L_\nu^4}{m \mu_\nu^2} \Exp[(F - \xi_\nu(\theta))^2] +
    \frac{4S^2(\max(\nu,1-\nu))^2F_{\max}^2}{m}
    \right) \nonumber \\
    &+ \frac{16(\max(\nu,1-\nu))^2F_{\max}^2S^2}{\delta^4}\left(\frac{(2\nu -1)^2}{m} + \frac{(1-2\nu)^2 p_{\max}^2 L_\nu^2}{m \mu_\nu^2}  \mathbb{E}_\theta[(F -\xi_\nu(\theta))^2] \right) \nonumber \\
    & \le \Bigg[\frac{4}{\delta^2} \left( \frac{S^2 L_\nu^4}{ \mu_\nu^2} \Exp[(F - \xi_\nu(\theta))^2] +
    4S^2(\max(\nu,1-\nu))^2F_{\max}^2
    \right) \nonumber \\
    &+ \frac{16(\max(\nu,1-\nu))^2F_{\max}^2S^2}{\delta^4}\left((2\nu -1)^2 + \frac{(1-2\nu)^2 p_{\max}^2 L_\nu^2}{ \mu_\nu^2}  \mathbb{E}_\theta[(F -\xi_\nu(\theta))^2] \right)\Bigg] \times \frac{1}{m}. \nonumber 
\end{align}
Using \Cref{asm:boundedness}, $
\sup_{\theta\in\Theta}
\mathbb{E}_\theta\!\left[(F-\xi_\nu(\theta))^2\right]
\le 4F_{\max}^2.$
Using this fact in the final inequality above, we have
\[
\mathbb{E}_\theta\!\left[
  \left\|\widehat{\nabla\xi}_{\nu,\theta}^{\,m}
       -\nabla\xi_\nu(\theta)\right\|^2
\right]
\le \frac{C}{m},
\]
where $C$ is as defined in the theorem statement.
\end{proof}

\section{Proof of Theorem \ref{theorem:non_asymptotic_bound}}  
\label{proof:non_asymptotic_bound}  

We follow a technique similar to that used in \cite{liu2018zeroth,vijayan_smoothed_2021,j2016proximal} to derive a non-asymptotic bound for our proposed method. Before proving Theorem \ref{theorem:non_asymptotic_bound}, we state some properties of $\mathcal{P}_{\Theta}$. 
\begin{lemma} 
The projection operator $\mathcal{P}_{\Theta} $ satisfies the following properties: 
\begin{enumerate}
    \item $\|\mathcal{P}_{\Theta}\left(\theta, f(\theta),\eta \right)\| \leq \|f(\theta)\| $,   
    \item $\|\mathcal{P}_{\Theta}\left(\theta, f(\theta),\eta \right) - \mathcal{P}_{\Theta}\left(\theta, g(\theta),\eta \right)\| \leq \|f(\theta) - g(\theta)\|$,  
    \item $\left<f(\theta), \mathcal{P}_{\Theta}\left(\theta, f(\theta),\eta \right)\right> \geq \|\mathcal{P}_{\Theta}\left(\theta, f(\theta),\eta \right)\|^2 $. 
\end{enumerate}

\end{lemma}
\begin{proof}
    See Lemma 10 in \cite{vijayan_smoothed_2021}. 
\end{proof}
\begin{proof}(Theorem \ref{theorem:non_asymptotic_bound})
From the fundamental theorem of calculus, we have 
\begin{align*}
 &
    \xi_{\nu}(\theta_{k+1}) - \xi_{\nu}(\theta_{k}) \\
    & = \left<\nabla \xi_{\nu}(\theta_k), \theta_{k+1} - \theta_{k} \right> + \int_0^1 \left< \nabla \xi_{\nu}(\theta_{k+1} + \tau (\theta_{k} - \theta_{k+1})) - \nabla \xi_{\nu}(\theta_k), (\theta_{k+1} - \theta_{k})\right> d\tau \\
    & \overset{(a)}{\leq} \left<\nabla \xi_{\nu}(\theta_k), \theta_{k+1} - \theta_{k} \right> + \int_0^1  \|\nabla \xi_{\nu}(\theta_{k+1} + \tau (\theta_{k} - \theta_{k+1})) - \nabla \xi_{\nu}(\theta_k)\|.\|\theta_{k+1} - \theta_{k}\| d\tau \\ 
    & \overset{(b)}{\leq} \left<\nabla \xi_{\nu}(\theta_k), \theta_{k+1} - \theta_{k}\right> + \frac{L_{\xi}}{2} \|\theta_{k+1} - \theta_{k}\|^2, \; \; \; \textrm{by Lipchitz condition}\\
    & \overset{(c)}{=} -\eta  \left<\nabla \xi_{\nu}(\theta_k), \mathcal{P}_{\Theta}\left(\theta_k, \widehat{\nabla} \xi_{\nu}(\theta_k),\eta_k \right)\right> + \frac{L_{\xi}}{2} \eta^2 \|\mathcal{P}_{\Theta}\left(\theta_k, \widehat{\nabla }\xi_{\nu}(\theta_k),\eta_k \right)\|^2\\ 
    &  \overset{(d)}{=} -\eta \left[\left<\widehat{\nabla }\xi_{\nu}(\theta_k), \mathcal{P}_{\Theta}\left(\theta_k, \widehat{\nabla }\xi_{\nu}(\theta_k),\eta_k \right)\right>  - \left< - \nabla \xi_{\nu}(\theta_k)+ \widehat{\nabla }\xi_{\nu}(\theta_k),  \mathcal{P}_{\Theta}\left(\theta_k, \widehat{\nabla }\xi_{\nu}(\theta_k),\eta_k \right)\right>\right]  \\
    & \qquad \qquad \qquad+ \frac{L_{\xi}}{2} \eta^2 \|\mathcal{P}_{\Theta}\left(\theta_k, \widehat{\nabla }\xi_{\nu}(\theta_k),\eta_k \right)\|^2\\  
    & \overset{(e)}{\leq} - \frac{\eta}{2}\|\mathcal{P}_{\Theta}\left(\theta_k, \widehat{\nabla }\xi_{\nu}(\theta_k),\eta_k \right)\|^2  + \frac{\eta}{2} \|\nabla \xi_{\nu}(\theta_k) - \widehat{\nabla}\xi_{\nu}(\theta_k)\|^2 + \frac{L_{\xi}}{2} \eta^2 \|\mathcal{P}_{\Theta}\left(\theta_k, \widehat{\nabla }\xi_{\nu}(\theta_k),\eta_k \right)\|^2 \\
    & = -\frac{\eta}{2} (1- \eta L_{\xi}) \|\mathcal{P}_{\Theta}\left(\theta_k, \widehat{\nabla }\xi_{\nu}(\theta_k),\eta_k \right)\|^2  + \frac{\eta}{2} \|\nabla \xi_{\nu}(\theta_k) - \widehat{\nabla}\xi_{\nu}(\theta_k)\|^2.  
\end{align*} 

We obtain 
\begin{align}
     \|\mathcal{P}_{\Theta}\left(\theta_k, \widehat{\nabla }\xi_{\nu}(\theta_k),\eta_k \right)\|^2\leq \frac{2}{\eta (1-\eta L_{\xi})}[\xi_{\nu}(\theta_{k}) -\xi_{\nu}(\theta_{k+1}) ] +\frac{1}{ (1-\eta L_{\xi})} \|\nabla \xi_{\nu}(\theta_k) - \widehat{\nabla}\xi_{\nu}(\theta_k)\|^2. 
\end{align} 

Now
\begin{align*}
&
    \|\mathcal{P}_{\Theta}\left(\theta_k,\nabla \xi_{\nu}(\theta_k),\eta_k \right)\|^2 \\
    & \overset{(a)}{\leq} 2\|\mathcal{P}_{\Theta}\left(\theta_k, \widehat{\nabla }\xi_{\nu}(\theta_k),\eta_k \right)\|^2  + 2 \|\mathcal{P}_{\Theta}\left(\theta_k, \widehat{\nabla }\xi_{\nu}(\theta_k),\eta_k \right) - \mathcal{P}_{\Theta}\left(\theta_k,\nabla \xi_{\nu}(\theta_k),\eta_k \right)\|^2 \\
    & \overset{(b)}{\leq} 2\|\mathcal{P}_{\Theta}\left(\theta_k, \widehat{\nabla }\xi_{\nu}(\theta_k),\eta_k \right)\|^2 + 2 \| \widehat{\nabla }\xi_{\nu}(\theta_k) - \nabla \xi_{\nu}(\theta_k)\|^2 \\
    & \overset{(c)}{\leq} \frac{4}{\eta (1-\eta L_{\xi})} [\xi_{\nu}(\theta_{k}) -\xi_{\nu}(\theta_{k+1}) ] + \left(\frac{2}{1-\eta L_{\xi}} + 2\right)\|\nabla \xi_{\nu}(\theta_k) - \widehat{\nabla}\xi_{\nu}(\theta_k)\|^2 
\end{align*} 
Taking the expectation of both sides with respect to the sigma field $\mathcal{F}_k$, we get 
\begin{align*}
    \Exp [\|\mathcal{P}_{\Theta}\left(\theta_k, \nabla \xi_{\nu}(\theta_k),\eta_k \right)\|^2] &\leq \frac{4}{\eta (1-\eta L_{\xi})} \Exp [\xi_{\nu}(\theta_{k}) - \xi_{\nu}(\theta_{k+1})] \\
    & + \left(\frac{2}{1-\eta L_{\xi}} + 2\right)\Exp [\|\nabla \xi_{\nu}(\theta_k) - \widehat{\nabla} \xi_{\nu}(\theta_k)\|^2]  \\ 
    & \leq \frac{4}{\eta (1-\eta L_{\xi})} \Exp [\xi_{\nu}(\theta_k) - \xi_{\nu}(\theta_{k+1})] + \left(\frac{2}{1-\eta L_{\xi}} + 2\right) \frac{C}{m}. 
\end{align*}   
Running the summation from $k=1,2,\ldots, N$, we have
\begin{align} \label{eq:bound_estimator}
    \sum_{k=1}^{N} \Exp [\|\mathcal{P}_{\Theta}\left(\theta_k, \nabla \xi_{\nu}(\theta_k),\eta_k \right)\|^2] \leq \sum_{k=1}^{N} \left[\frac{4}{\eta (1-\eta L_{\xi})} \Exp [\xi_{\nu}(\theta_k) - \xi_{\nu}(\theta_{k+1})] + \left(\frac{2}{1-\eta L_{\xi}} + 2\right) \frac{C}{m}\right]. 
\end{align} 
Note that
\begin{align*}
    \Exp [\|\mathcal{P}_{\Theta}\left(\theta_R, \nabla \xi_{\nu}(\theta_R),\eta_R \right)\|^2] &= \frac{1}{N}\sum_{k=1}^{N}\Exp[\|\mathcal{P}_{\Theta}\left(\theta_k, \nabla \xi_{\nu}(\theta_k),\eta_k \right)\|^2] \\
    & \overset{(a)}{\leq} \frac{4[\xi_{\nu}(\theta_1) - \xi_{\nu}(\theta_{\exp}^*)]}{\eta (1-\eta L_{\xi}) N} + (\frac{2}{1-\eta L_{\xi}}+2)\frac{C}{m},  
\end{align*} 
where $R$ is chosen uniformly at random from $\{1,2,\ldots, N\}$ and (a) follows from \eqref{eq:bound_estimator} above. 

For $\eta  =  \frac{1}{\sqrt{N}} \leq \frac{1}{2L_{\xi}}$ and $m = \sqrt{N}$, we obtain
    \begin{align*}
        \Exp\left[
        \left\|
            \mathcal{P}_\Theta\left(
                \theta_R,\nabla\xi_\nu(\theta_R),\eta_R
            \right)
        \right\|^2
    \right]\leq \frac{8[\xi_{\nu}(\theta_1) - \xi_{\nu}(\theta_{\exp}^*)]}{\sqrt{N}} + \frac{6C}{\sqrt{N}}.  
    \end{align*} 
\end{proof}

\section{Additional experimental details}  
\label{sec:add_experimt} 
\begin{table}[ht]
    \centering
    \caption{Hyperparameters used for training on three MuJoCo environments }
    \label{tab:hyperparameters} 
    \begin{tabular}{|l|r|r|r|} 
        \toprule
        \textbf{Method} & \textbf{Reacher-v5} & \textbf{Hopper-v5} &\textbf{InvertedPendulum-v5 } \\
        \midrule
        Hidden size  $(n)$         &32 &32 &32 \\
        Discount factor  $(\gamma)$         &0.99 &0.99 &0.99 \\
        Learning rate $(\eta)$          & 0.002 & 0.002 & 0.003\\
        Episodes $(N)$          & 20000 & 20000 & 10000\\
        Batch size 
        $(B)$          & 16 & 16 & 30\\
        Optimizer           & Adam & Adam & Adam \\
        \bottomrule
    \end{tabular}
\end{table}  
\begin{figure}
    \centering
    \includegraphics[width=0.99\linewidth]{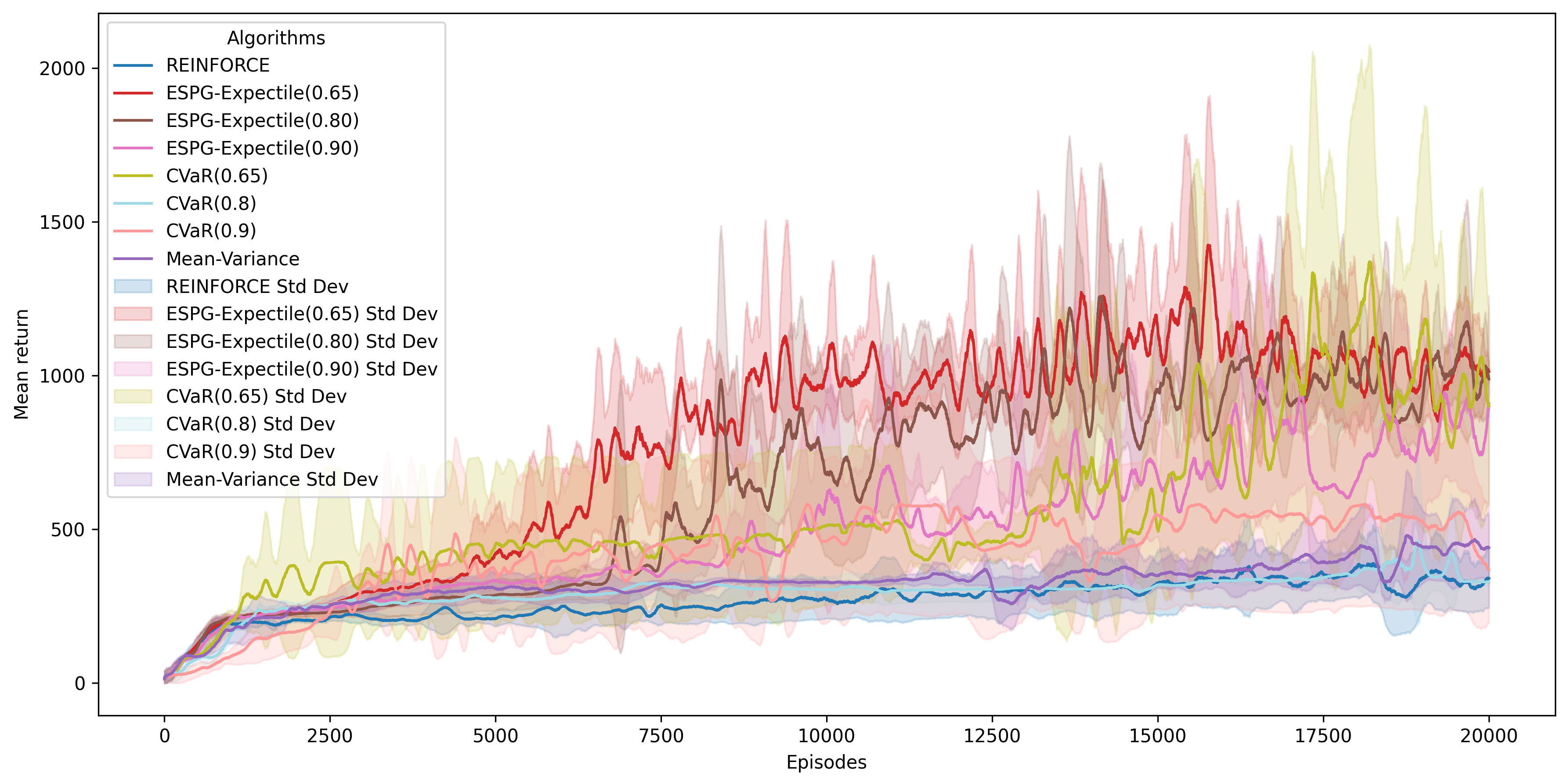} 
    \caption{Training curve demonstrating the performance on the Hopper environment. }
    \label{fig:training_hopper}
\end{figure} 
\textbf{Model details: }
For all methods except REINFORCE, we convert rewards to costs by negating them, then apply the ESPG algorithm to the resulting costs. For all algorithms, we used a two-layer neural network, with $n$ neurons per layer, for the policy network. We trained each instance in $N$ episodes with a batch size of $B$, where $n, N, B$ are specified in \Cref{tab:hyperparameters} for each environment. Note that the discount factor $(\gamma) = 0.99$ and the Adam optimizer are commonly used for the MuJoCo environment, cf. \cite{fujimoto2018addressing}, \cite{haarnoja2018soft}. For other hyperparameters such as learning rate $(\eta)$, hidden size, and batch size, we conducted an extensive grid search and selected the best-performing values. We use the default horizon length for each environment in our experiment. 

Post-training, we fixed the trained policy and sampled 250 episodes from each seed for each algorithm.

\textbf{Compute resources:} We conducted all the experiments on a server with the model name AMD EPYC 7763 64-Core Processor, which contains $128$ CPUs based on $x86\_64$ architecture with 503GB RAM. Every experiment took on average 3 hours to run using parallelisation technique.  

\begin{figure}[h]
	\centering
	\begin{subfigure}{0.8\textwidth}
		\includegraphics[width=\textwidth]{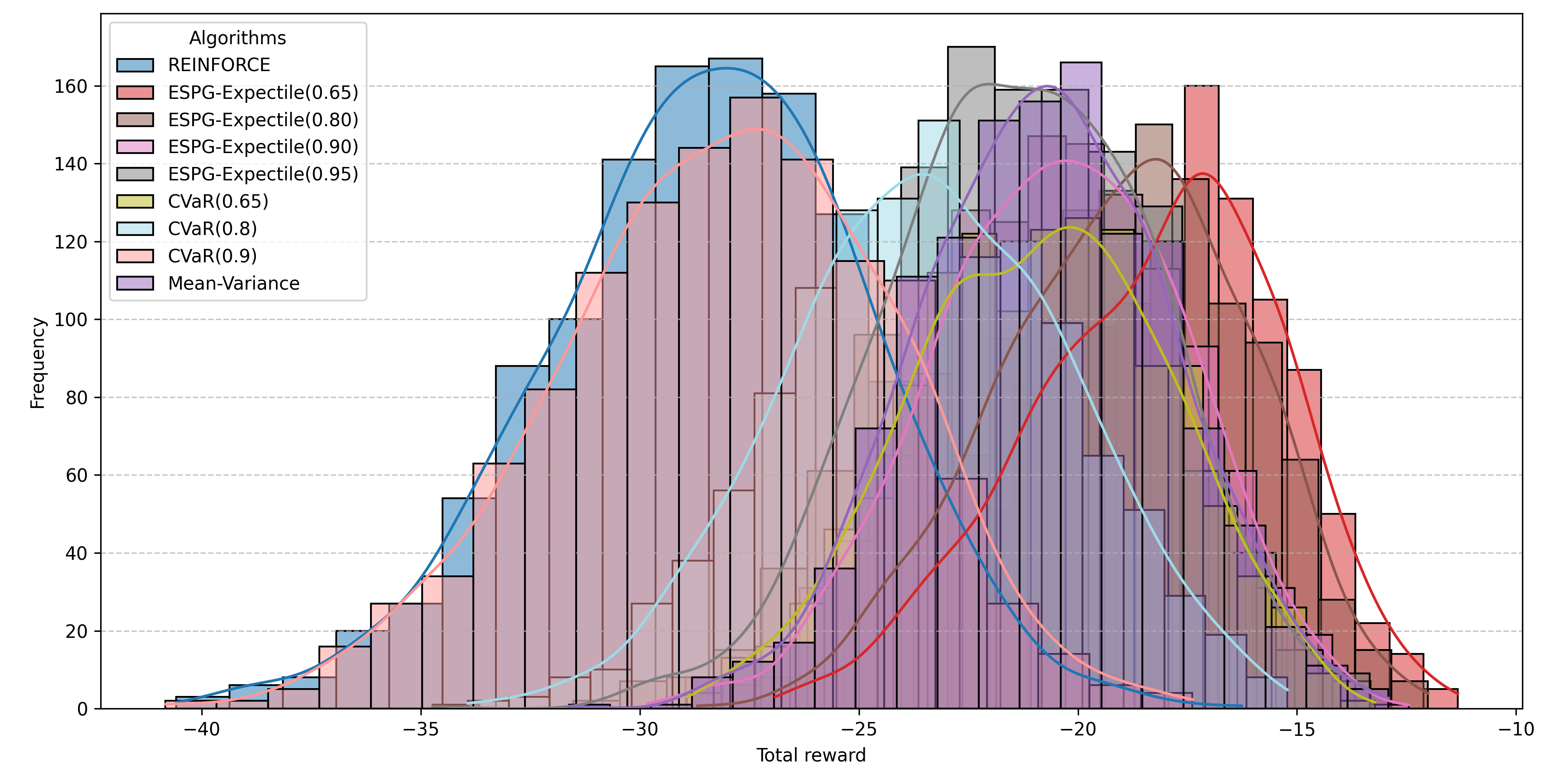}
		\caption{Return distribution on Reacher environment}
        \label{fig:Hist_recaher}
	\end{subfigure}
    \\
	\begin{subfigure}{0.8\textwidth}
		\includegraphics[width=\textwidth]{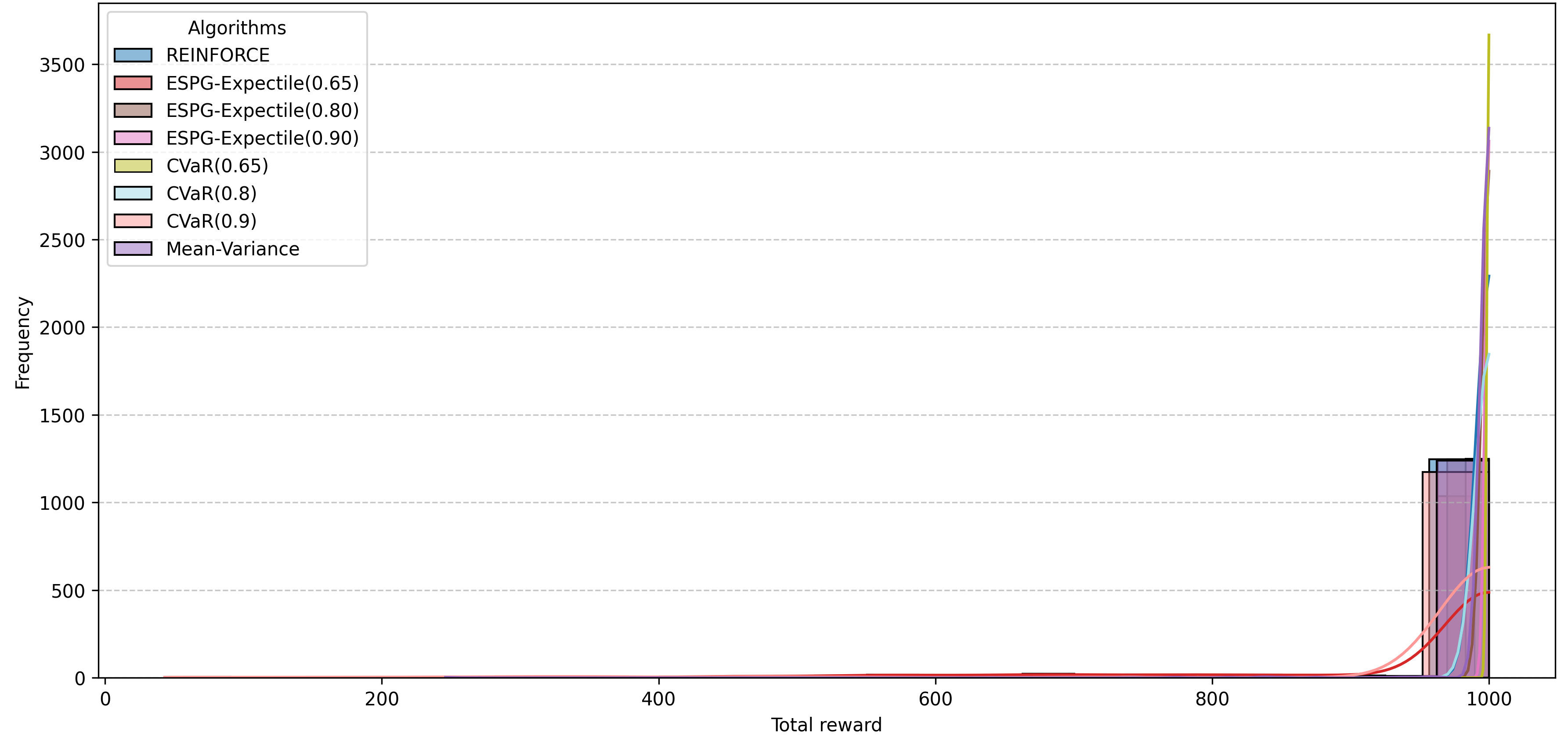}
		\caption{Return distribution on InvertedPendulum}
        \label{fig:Hist_invtpendulm}
	\end{subfigure} \\
\begin{subfigure}{0.8\textwidth}
		\includegraphics[width=\textwidth]{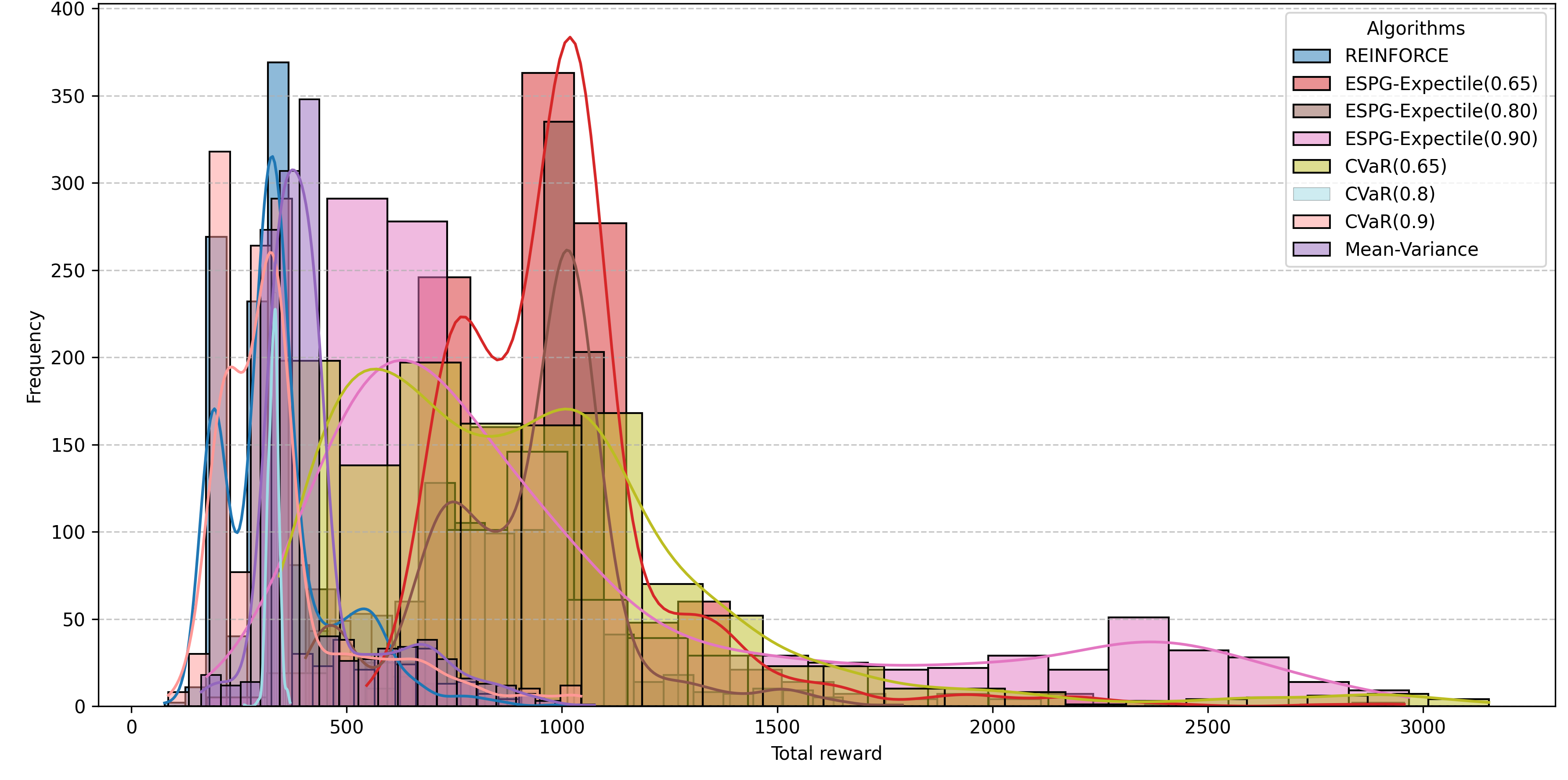}
		\caption{Return distribution on Hopper }
        \label{fig:Hist_hopper}
	\end{subfigure} 
	\caption{Return distribution on three environments}  
    \label{fig:Hist_all} 
\end{figure} 
In \Cref{fig:training_hopper}, we plot the training curve of mean return on the Hopper environment. We plot the empirical return distributions for all methods in \Cref{fig:Hist_all} on MuJoCo environments.


\end{document}